\documentclass{article}

\usepackage[round]{natbib}
\usepackage{times}
\usepackage{fancyhdr}

\usepackage[parfill]{parskip}

\usepackage{microtype}
\usepackage{graphicx}
\usepackage{subfig}
\usepackage{booktabs} 

\usepackage{fullpage}
\usepackage{courier}
\usepackage{authblk}

\usepackage[utf8]{inputenc} 
\usepackage[T1]{fontenc}    
\usepackage[hidelinks]{hyperref}       

\usepackage{amsmath}
\usepackage{amssymb}
\usepackage{mathtools}
\usepackage{amsthm}

\usepackage{url}            
\usepackage{booktabs}       
\usepackage{amsfonts}       
\usepackage{amssymb}       
\usepackage{nicefrac}       
\usepackage{microtype}      
\usepackage{graphicx}
\usepackage{xcolor}
\usepackage{enumitem}
\usepackage{xspace}
\usepackage[export]{adjustbox}

\usepackage{algorithm}
\usepackage{algorithmic}

\usepackage{multirow}
\usepackage{multicol}

\usepackage{floatpag}

\usepackage{color,array}

\theoremstyle{plain}
\newtheorem{theorem}{Theorem}[section]
\newtheorem{proposition}[theorem]{Proposition}
\newtheorem{lemma}[theorem]{Lemma}
\newtheorem{corollary}[theorem]{Corollary}
\newtheorem{definition}[theorem]{Definition}

\theoremstyle{remark}
\newtheorem{remark}[theorem]{Remark}

\DeclareMathOperator*{\argmax}{arg\,max}
\DeclareMathOperator*{\argmin}{arg\,min}

\newcommand{\ppn}{\hat{P}_{n}}

\newcommand{\dd}{\mathrm{d}}

\title{Functional Generalized Empirical Likelihood Estimation for\\Conditional Moment Restrictions}

\author[$1$]{Heiner Kremer}
\author[$2$]{Jia-Jie Zhu}
\author[$1$]{Krikamol Muandet}
\author[$1,3$]{Bernhard Sch{\"o}lkopf}

\affil[$1$]{Max Planck Institute for Intelligent Systems, Tübingen, Germany}
\affil[$2$]{Weierstrass Institute for Applied Analysis and Stochastics, Berlin, Germany}
\affil[$3$]{Eidgen\"ossische Technische Hochschule Zürich, Switzerland}

\begin{document}
\floatpagestyle{plain}
\date{}
\maketitle

\begin{abstract}
\noindent
Important problems in causal inference, economics, and, more generally, robust machine learning can be expressed as conditional moment restrictions, but estimation becomes challenging as it requires solving a continuum of unconditional moment restrictions. Previous works addressed this problem by extending the generalized method of moments (GMM) to continuum moment restrictions. In contrast, generalized empirical likelihood (GEL) provides a more general framework and has been shown to enjoy favorable small-sample properties compared to GMM-based estimators. To benefit from recent developments in machine learning, we provide a functional reformulation of GEL in which arbitrary models can be leveraged. Motivated by a dual formulation of the resulting infinite dimensional optimization problem, we devise a practical method and explore its asymptotic properties. Finally, we provide kernel- and neural network-based implementations of the estimator, which achieve state-of-the-art empirical performance on two conditional moment restriction problems.
\end{abstract}

\section{Introduction}
Moment restrictions identify a parameter of interest by restricting the expectation value of so-called moment functions, which depend on the parameter and random variables representing the underlying noisy data generating process. Important problems in causal inference, economics, and generally robust machine learning can be cast in this form \citep{Newey93:CMR,ai2003efficient,bennett2020variational,Dikkala20:Minimax}. 
Particularly challenging are problems formulated as \emph{conditional} moment restrictions (CMR), which constrain the conditional expectation of the moment function. Such problems appear, e.g., in instrumental variable (IV) regression \citep{newey2003instrumental,angrist2008mostly}, where the expectation of the residual of the prediction conditioned on so-called instruments is restricted to be zero.
Other applications are policy learning \citep{bennett2020efficient} and off-policy evaluation in reinforcement learning \citep{kallus2020double,bennett2021off,chen2021instrumental} as well as double/debiased machine learning \citep{chernozhukov2016double,chernozhukov2017double,chernozhukov2018double}.
As conditional moment restrictions are difficult to handle directly, a common approach is to transform them into an infinite number of corresponding unconditional moment restrictions \cite{BIERENS1982105}. Generalizing the corresponding estimation methods from the finite-dimensional case to the infinite-dimensional case is an active area of research \citep{Carrasco1,Carrasco2,chausse2012generalized,carrasco2017regularized,muandet2020kernel,bennett2020variational,zhang2021maximum}.
One of the most popular approaches to learning with moment restrictions is Hansen's celebrated generalized method of moments (GMM) \citep{hansen}.
In order to improve the small sample properties of GMM estimators, alternative methods have been proposed and are generally known as generalized empirical likelihood (GEL) estimators \citep{smith1997alternative,smith2011,newey04}. GEL generalizes the original empirical likelihood framework developed by \citet{owen88,owen90} and \citet{qin-lawless} to different divergence functions and contains many related estimators as special cases. While closely related to GMM, the estimators from the GEL family have been shown theoretically to exhibit smaller higher-order-biases than those of GMM \citep{newey04} and therefore promise to have favorable small sample properties. 
With increasing number of over-identifying restrictions, i.e., when the number of restrictions exceeds the number of parameters, this advantage has been shown theoretically to become more significant \citep{newey02,donald2003}. 
Therefore, we expect the framework to be particularly suited for the case of infinitely many restrictions. We leverage this potential for conditional moment restrictions by developing the theoretical foundation for a GEL framework with continua of moment restrictions.
 \paragraph{Our contributions} 
First, we extend the GEL framework to conditional moment restrictions by generalizing it to \emph{functional} moment restrictions. 
Second, building on a result from infinite optimization, we derive a dual form which allows us to employ modern machine learning models in the GEL context. This generalizes existing results not only to functional moment restrictions but also to general $\varphi$-divergences beyond the Cressie-Reed family. 
Third, we provide the asymptotic properties of our estimator, showing that it is consistent for functional moment restrictions. Then, we show how the result for the functional case can be translated to the conditional case yielding an efficient estimator for conditional moment restrictions.
Finally, we discuss the relation to existing methods and provide experimental results.
Compared to previous extensions of GEL \citep{kitamura2004,tang2010penalized,chausse2012generalized, carrasco2017regularized}, our approach combines the idea of a continuum generalization of GEL \citep{chausse2012generalized,carrasco2017regularized} with the flexibility of machine learning models such as neural networks and kernel methods.
Our general framework contains related estimators such as a one-step/continuous updating version of the variational method of moments (VMM) estimator \citep{bennett2020variational} as special cases. In contrast to VMM, our method allows the use of divergences other than the $\chi^2$-divergence.
The remainder of this paper is organized as follows. Section~\ref{sec:1:moment-restrictions} introduces the method of moments framework \citep{hall2004generalized} and two popular relaxations. 
Section~\ref{sec:1:FGEL} presents our main contributions, the theoretical development of our FGEL estimator, followed by experimental results in Section~\ref{sec:1:experiments}.
Finally, we discuss related works in Section~\ref{sec:1:related-work}.
Compared to the initial conference version, this paper has been updated using recent results of \citet{pmlr-v202-kremer23a} to derive an explicit convergence rate for the estimator, as well as to prove its semi-parametric efficiency.
\section{Learning with Moment Restrictions} \label{sec:1:moment-restrictions}
Let $X$ be a random variable taking values in $\mathcal{X}\subseteq \mathbb{R}^r$ with distribution $P_0$ and let $\psi: \mathcal{X} \times \Theta \rightarrow \mathbb{R}^m$ denote a vector of $m$ functions, the so-called moment functions, with parameters $\theta \in \Theta \subset \mathbb{R}^{p}$. We denote with $E_P[\cdot]$ the expectation over all random variables that are not conditioned on with respect to a distribution $P$ and refer to the population distribution $P_0$ whenever we omit the subscript.
Assume that there exists a unique parameter $\theta_0 \in \Theta$ such that $E[\psi(X;\theta_0)] = 0$.
For instance, $E[X - \theta_0] = 0$ characterizes the mean of $P_0$. Our goal is to estimate $\theta_0$ based on a sample $\{ x_i \}_{i=1}^n$ from $P_0$. 
The corresponding empirical moment restrictions read
\begin{align}
    E_{\hat{P}_n}[\psi(X;\theta)] = 0, \quad \theta\in\Theta,
    \label{eq:1:empirical-moment-conditions}
\end{align}
where $\hat{P}_n = \sum_{i=1}^{n}  \frac{1}{n} \delta_{x_i}$ is the empirical distribution.
This is a system of $m$ estimating equations for $p$ parameters which can be fulfilled exactly as long as $m \leq p$.
For example, $E_{\hat{P}_n}[X - \theta] = 0$ gives $\theta = \frac{1}{n}\sum_{i=1}^nx_i$ as an empirical estimate of the mean of $P_0$.
However, in the over-identified case, i.e., when the number of non-redundant moment restrictions exceeds the number of parameters ($m>p$), 
it is generally impossible to fulfill all moment restrictions \eqref{eq:1:empirical-moment-conditions} exactly. 
To obtain a feasible problem, the constraints \eqref{eq:1:empirical-moment-conditions} need to be relaxed. 
Below we discuss two popular approaches, namely, the generalized method of moments \citep{hansen} and maximum (generalized) empirical likelihood estimation \citep{owen88,owen90, qin-lawless}.
\paragraph{Generalized Method of Moments (GMM)}
The generalized method of moments relaxes the constraint \eqref{eq:1:empirical-moment-conditions} into a minimization of a quadratic form of the average of the moment functions, i.e., $\theta^{\mathrm{GMM}}_W = \argmin_{\theta\in\Theta} \, \hat{\psi}(\theta)^{T}W\hat{\psi}(\theta)$, where $\hat{\psi}(\theta) := E_{\hat{P}_{n}}[\psi(X;\theta)]$ and $W \in \mathbb{R}^{m \times m}$ denotes the so-called weighting matrix. Asymptotic normality theory shows that an efficient estimator, i.e., an estimator with minimal asymptotic variance among the class of GMM estimators, is obtained by choosing $W$ as the inverse covariance matrix of the moment functions, $W = \widehat{\Omega}_\theta^{-1}$, where $\widehat{\Omega}_\theta := E_{\hat{P}_n}[\psi(X;\theta) \psi(X;\theta)^T]$, which itself a function of $\theta$~\citep{hansen}. The resulting estimator, i.e.,
\begin{align}
    \theta^{\mathrm{CUE}} = \argmin_{\theta\in\Theta} \; \hat{\psi}(\theta)^{T} \widehat{\Omega}_\theta^{-1} \hat{\psi}(\theta),
    \label{eq:1:cue}
\end{align}
is the continuous updating estimator (CUE) of \citet{hansen-finite-samples} which results from a non-convex optimization problem and can exhibit unfavorable convergence properties if $\widehat{\Omega}_\theta$ is ill-conditioned \citep{hall2004generalized}.
Therefore, one often resorts to a 2-step procedure: first, an inefficient but consistent estimate $\tilde{\theta}$ of $\theta_0$ is obtained, e.g., by setting $W=I$. Second, this estimate is used to compute $\widehat{\Omega}_{\tilde{\theta}}^{-1}$ which is kept fixed during the second optimization step. 
This yields the so-called optimally weighted GMM estimator \citep{hansen},
\begin{align}
    \theta^{\mathrm{OWGMM}} = \argmin_{\theta\in\Theta}\; \hat{\psi}(\theta)^{T} \widehat{\Omega}_{\tilde{\theta}}^{-1} \hat{\psi}(\theta).
    \label{eq:1:owgmm}
\end{align}
A more in-depth exposition of the GMM framework can be found in \citet{hall2004generalized}. 
\paragraph{Generalized Empirical Likelihood (GEL)}
The empirical likelihood framework \citep{owen88, owen90,qin-lawless} relaxes the restrictions \eqref{eq:1:empirical-moment-conditions} by requiring $E_P[\psi(X;\theta)]=0$ to be fulfilled exactly but allowing the distribution $P$ to deviate from the empirical distribution $\hat{P}_n$. For a continuous function $\varphi: \mathbb{R} \rightarrow \mathbb{R}$ we define the $\varphi$-divergence between distributions $P$ and $Q$
as 
$D_\varphi(P||Q)= \int \varphi \big( \frac{ \dd P}{ \dd Q}\big) \dd Q$, 
where $\frac{\dd P}{\dd Q}$ denotes the Radon-Nikodym derivative of $P$ with respect to $Q$.
Then, we can define the profile divergence with respect to this $\varphi$-divergence as
\begin{align}
    R(\theta) = \inf_{P \ll \hat{P}_n} \ D_\varphi(P||\hat{P}_{n}) \quad \mathrm{s.t.} \quad E_{P}[\psi(X; \theta)] =0,  \quad E_{P}[1]=1,
\end{align}
where $P \ll \hat{P}_n$ describes the set of positive measures $P$ that are absolutely continuous with respect to the empirical distribution $\hat{P}_n$. 
In other words, $P$ describes multinomial distributions on the sample, i.e., re-weightings of the data points. The maximum empirical likelihood estimator (MELE) for $\theta$ is then given by $\theta^{\mathrm{EL}} = \argmin_{\theta \in \Theta} R(\theta)$.
The framework, originally proposed as empirical likelihood by Owen for the case $\varphi(p) = -2 \log(p)$, has been generalized to other divergence measures for which it is known as minimum discrepancy (MD) \citep{corcoran98} or generalized empirical likelihood \citep{smith1997alternative,smith2005local}. 
The latter corresponds to its dual formulation. It contains many related estimators as special cases. 
For example, by choosing the function $\varphi$ from the Cressie-Read family of non-parametric discrepancy measures \citep{cressie-read},
\begin{align}
    \varphi_{\gamma}(p) = \frac{1}{\gamma(\gamma+1)} \left( p^{\gamma+1}-1 \right),
    \label{eq:1:cressie}
\end{align}
one retrieves the CUE for $\gamma=1$ \citep{newey04}, the exponential tilting estimator for $\gamma \rightarrow 0$ \citep{kitamura} and finally the original empirical likelihood estimator for $\gamma \rightarrow -1$ \citep{qin-lawless}. 
Detailed exposition of the GEL framework can be found in \citet{smith1997alternative} and \citet{owen2001empirical}.
\section{Functional Generalized Empirical Likelihood}\label{sec:1:FGEL}
In this work, we are concerned with problems that can be expressed by infinitely many moment restrictions, especially those that arise from \emph{conditional} moment restrictions (CMR)~\citep{Newey93:CMR,ai2003efficient} of the form 
\begin{align}
    E[\psi(X;\theta_0) | Z] = 0, \quad  P_Z\textrm{-a.s.}, \label{eq:1:conditional}
\end{align}
where $Z \in \mathcal{Z}$ is an additional random variable with marginal distribution $P_Z$. 
By the law of iterated expectation, the CMR \eqref{eq:1:conditional} can be expressed in terms of infinitely many unconditional moment restrictions \citep{BIERENS1982105} 
\begin{align}
    E[\psi(X;\theta_0)^{T}h(Z)] = 0, \quad \forall h \in \mathcal{H}, \label{eq:1:unconditional}
\end{align}
where $\mathcal{H}$ denotes a space of measurable functions $h: \mathcal{Z} \rightarrow \mathbb{R}^m$, i.e., the space of square integrable functions $L^2(P_Z)$.
As \eqref{eq:1:unconditional} has to hold for all functions in $\mathcal{H}$, this implies an uncountable infinite number, i.e., a continuum, of moment restrictions ($m=\infty$). 
For example, the instrumental variable regression problem can be described by a CMR via $E[Y - f(X; \theta_0) | Z] = 0$ where $Z$ is an instrumental variable and $\theta \in \Theta$ parameterizes a function $f: \mathcal{X} \rightarrow \mathcal{Y}$. Motivated by this example, in the following, we will refer to $Z$ and $h$ as instrument and instrument function, respectively, in the context of general CMR.
\subsection{Our Method}
Maximum empirical likelihood estimation is based on minimizing a profile divergence $R: \Theta \rightarrow \mathbb{R}$ over a parameter space $\Theta$. Let $\mathcal{P} := \{P \ll \hat{P}_n : E_{P}[1]= 1 \}$ denote the set of distributions that are absolutely continuous with respect to the empirical distribution.
For conditional moment restrictions of the form \eqref{eq:1:conditional}, we can define the profile divergence as 
\begin{align}
    R(\theta) :=  \min_{P \in \mathcal{P}} D_\varphi(P || \hat{P}_n) \quad  \mathrm{s.t.} \quad
    E_{P}[\psi(X;\theta) | Z]= 0 \quad  P_Z\textrm{-a.s.}, \label{eq:1:123}
\end{align}
where $D_\varphi$ is defined in terms of the $\varphi$-divergence (see Table \ref{table:divergencies}).
Let $\mathcal{H}$ be a sufficiently large Hilbert space of functions such that \eqref{eq:1:unconditional} implies \eqref{eq:1:conditional} and let $\mathcal{H}_1 = \{ h\in\mathcal{H} : \| h\| \leq 1\}$ denote the unit ball in $\mathcal{H}$. 
Let $\mathcal{H}^\ast$ be the corresponding dual space of functionals $\mathcal{H} \rightarrow \mathbb{R}$ equipped with the dual norm $\| \cdot \|_{\mathcal{H}^\ast}$ defined for $\Psi \in \mathcal{H}^\ast$ as $ \| \Psi \|_{\mathcal{H}^\ast} = \sup_{h \in \mathcal{H}_1} \Psi(h)$. Then, we can define the \emph{moment functional}, a statistical functional $\Psi(X,Z;\theta) \in \mathcal{H}^\ast$, as
\begin{align*}
    \Psi(X,Z;\theta): \ \ & \mathcal{H} \rightarrow \  \mathbb{R} \\
     & h \mapsto  \ \Psi(X,Z; \theta)(h) = \psi(X;\theta)^{T}h(Z),
\end{align*}
which can be seen as a weighted evaluation functional with respect to the conditioning variable $Z$.
With this definition, we can express \eqref{eq:1:unconditional} as the functional constraint $\| E_{P_0}[\Psi(X,Z; \theta_0)] \|_{\mathcal{H}^\ast} =0$. 
The computation of the profile likelihood thus becomes a \emph{functionally-constrained} optimization problem
\begin{align}
    R(\theta) =  \inf_{P\in \mathcal{P}}  D_\varphi(P||\hat{P}_n)  \quad \mathrm{s.t.} \quad \left\| E_{P}[\Psi(X,Z; \theta_0)] \right\|_{\mathcal{H}^\ast} =0. \label{eq:1:profile-divergence}
\end{align}
The FGEL problem arises from the dual formulation of \eqref{eq:1:profile-divergence}. For the case of finite dimensional moment restrictions, the duality relationship has been extensively explored by numerous works \citep{smith1997alternative,smith2011,kitamura2004,newey04}. 
However, as shown by \citet{borwein1993failure} these duality results do not carry over to infinite dimensional restrictions. 
Following the approach of \citet{borwein1993failure} and \citet{carrasco2017regularized}, we define a relaxed version of the functionally constrained profile likelihood \eqref{eq:1:profile-divergence} with relaxation parameter $\lambda >0$ as
\begin{align}
    R_\lambda(\theta) :=  \inf_{P\in \mathcal{P}}  D_\varphi(P||\hat{P}_n) \quad \mathrm{s.t.} \quad \left\|E_{P}[\Psi(X,Z;\theta)] \right\|_{\mathcal{H}^\ast} \leq \lambda. \label{eq:1:relaxed-profile-likelihood} 
\end{align}
With this relaxation, a constraint qualification condition holds and \eqref{eq:1:relaxed-profile-likelihood} admits a strongly dual form as formalized in the following theorem. 
\begin{theorem}\label{th:1:duality}
Let $\varphi^\ast(v) = \sup_{p \in \mathbb{R}^n} \langle v, p \rangle - \varphi(p)$ denote the Legendre-Fenchel conjugate function of a strongly convex function $\varphi$. Then the problem
\begin{align*}
    R_{\lambda}(\theta) = \inf_{p \in \mathbb{R}^n}  \sum_{i=1}^n \frac{1}{n} f(np_i) \quad \mathrm{s.t.} \quad \left\| \frac{1}{n}\sum_{i=1}^n p_i \Psi(x_i,z_i;\theta) \right\|_{\mathcal{H}^\ast} \leq \lambda, \ \ \sum_{i=1}^n p_i = 1
\end{align*}
admits the dual form 
\begin{align}
    R_{\lambda}(\theta) = \sup_{\substack{h \in \mathcal{H} \\ \mu \in \mathbb{R}}} \mu - \frac{1}{n} \sum_{i=1}^n \varphi^\ast \left( \Psi(x_i,z_i;\theta)(h) + \mu \right) - \lambda \| h \|_\mathcal{H}
    \label{eq:1:dual-form}
\end{align}
and strong duality holds between these formulations. Moreover, the unique minimizer of the primal problem is given by
$$p_{i} = \left( \frac{d}{dv}\varphi^\ast\right) \left(\Psi(x_i,z_i;\theta)(\hat{h}) + \hat{\mu}\right),$$ where $\hat{h}$, $\hat{\mu}$ are any solutions of the dual problem.
Moreover, as $\lambda \rightarrow 0$, $R_\lambda(\theta) \rightarrow R(\theta)$.
\end{theorem}
\begin{remark}
Theorem~\ref{th:1:duality} can be seen as a generalization of the duality result of \citet{newey04} not only to functional-valued moment restrictions but also to general strongly convex divergence functions beyond the Cressie-Reed family.
\end{remark}
Equation \eqref{eq:1:dual-form} provides a regularized functional generalization of the profile divergence. 
Based on this result, we define our functional generalized empirical likelihood estimator by making two modifications: 
first, we substitute the norm term in \eqref{eq:1:dual-form} for a differentiable quadratic version.
This modification is solely to simplify the analysis. Later we will choose the regularization parameter to be $\lambda_n = o_p(1)$ and find that $\|h\|_\mathcal{H} = o_p(1)$. Hence we can always find a $\chi'>0$ and $\lambda'_n = O_p(n^{-\chi'})$ such that $\lambda_n/2 \|h\|^2 \rightarrow 0$ and $\lambda'_n \|h\| \rightarrow 0$ at the same rate which implies that the formulations
are asymptotically equivalent.
Second, we drop the Lagrange parameter $\mu$ corresponding to the normalization constraint $\sum_{i=1}^n p_i = 1$. This is motivated by several observations.
From a theoretical point of view, it simplifies the problem at no cost as we will show later in Theorems~\ref{th:1:consistency-cmr}-\ref{efficiency-cmr} that setting $\mu=0$ still yields a consistent and efficient estimator. 
From a practical aspect it facilitates the implementation of the estimator using stochastic gradient methods. To see this, consider exemplarily the FGEL estimator with $\chi^2$-divergence. Then we can carry out the supremum over $\mu \in \mathbb{R}$ in closed form and obtain 
\begin{align*}
R^{\chi^2}_\lambda(\theta) =& \sup_{h} E_{\hat{P}_n}[\psi(X;\theta)^T h(Z)] \\
& - \frac{1}{2} E_{\hat{P}_n}\left[\left( \psi(X;\theta)^T h(Z) - E_{\hat{P}_n}[\psi(X;\theta)^T h(Z)] - 1\right)^2 \right] - \lambda \| h\|_\mathcal{H}.
\end{align*}
This contains two expectation operators combined in a non-linear way, which generally leads to biased gradient estimates and thus renders mini-batch stochastic gradient descent optimization complicated. In addition, by setting $\mu =0$, our estimator has the same form as the finite dimensional GEL estimator proposed by \citet{smith1997alternative} and \citet{kitamura} with the difference of involving an optimization over functions $h \in \mathcal{H}$ and an additional regularization term for $h$. 
With these modifications, we define our FGEL estimator as follows.

\begin{table}[t]
\centering
\caption{Common choices for the $\varphi$-divergence and the corresponding convex conjugate $\varphi^\ast(v) = \sup_{p} p^T v - \varphi(p)$ and its domain. A GEL function $\phi$ can be defined for each $\varphi$-divergence as $\phi(v) = - \varphi^\ast(v)$.}
\label{table:divergencies}
\begin{tabular}{ c c c  } 
    \toprule
     $\varphi(p)$ &  $\varphi^\ast(v)$ & $\operatorname{dom}(\varphi^\ast)$ \\
    \midrule
     $\frac{1}{2}(p -1)^2$ & $ \frac{1}{2}(1 + v)^2$ &  $\mathbb{R}$ \\
    $- \log(p)$ & $ \log(1 - v)$ & $\big(-\infty, 1 - \frac{1}{n}\big]$ \\
    $ p \log(p) $ & $ e^{v} $& $\mathbb{R}$ \\
    \bottomrule
\end{tabular}
\end{table}
\begin{definition}\label{def:1:gel-objective}
    Let $V\subseteq \mathbb{R}$ be an open interval containing zero and $\phi: V \rightarrow \mathbb{R}$ be a twice differentiable concave function with first and second derivatives $\phi_1(0) \neq 0$ and $\phi_2(0)<0$. Then we define the empirical FGEL objective $G: \Theta \times \widehat{\mathcal{H}}(\theta) \rightarrow \mathbb{R}$ as
    \begin{equation}
    G_{\lambda_n}(\theta, h) :=  
    \frac{1}{n} \sum_{i=1}^n \phi \left( \Psi(x_i,z_i;\theta)(h)\right) - \frac{\lambda_n}{2} \|h \|_{\mathcal{H}}^2, \label{eq:1:gel-objective} 
    \end{equation}
    where $\Psi(x_i,z_i;\theta)(h) = \psi(x_i;\theta)^{T}h(z_i)$ and
    $\widehat{\mathcal{H}}(\theta) := \{h \in \mathcal{H}: \   \psi(x_i;\theta)^{T}h(z_i) \in \operatorname{dom}(\phi), \  1\leq i \leq n \}$.
    The FGEL estimate $\hat{\theta}$ of $\theta_0$ results from a saddle point of $G_{\lambda_{n}}(\theta,h)$
    \begin{equation}
         \hat{\theta} = \argmin_{\theta \in \Theta} \sup_{h \in \widehat{\mathcal{H}}(\theta)} G_{\lambda_n}(\theta, h).
    \label{eq:1:regularized-fgel}
    \end{equation}
\end{definition}
\vspace{-0.5em}
\begin{remark}
    While the choice of $\phi$ can be motivated by the duality relationship of Theorem~\ref{th:1:duality}, i.e., $\phi = - \varphi^\ast$ with $\varphi^\ast(v) = \sup_p v^T p - \varphi(p)$ for $\varphi$ defined via some $\varphi$-divergence, GEL estimators can be defined for any function $\phi$ which fulfills the corresponding conditions of Definition~\ref{def:1:gel-objective}.
\end{remark}
Within the FGEL framework, the regularization term is responsible for regularizing an originally ill-posed operator estimation problem, which results from the optimization of $G_{\lambda_{n}}(\theta, h)$ over the instrument functions $h \in \mathcal{H}$.
We will demonstrate this here exemplarily for the $\chi^2$-divergence, which admits a closed form solution. Note that a similar argument has been provided earlier by \citet{carrasco2017regularized}.
Let $\Psi_{i}(\theta):= \Psi(x_i,z_i;\theta) \in \mathcal{H}^\ast$ and $\Psi_i^\ast \in \mathcal{H}$ denote its dual which can be identified with a function in $\mathcal{H}$ by the self-duality property of Hilbert spaces \citep{zeidler2012applied}. Then the first order condition for $h$ reads
\begin{align*}
    0&= - \frac{1}{n} \sum_{i=1}^n \left[  \Psi_i^\ast(\theta) - \left(\Psi_i^\ast(\theta) \Psi_i(\theta) + \lambda_n I \otimes I\right)(h) \right]\\
    & \Rightarrow h = - \left(\widehat{\Omega}_\theta + \lambda_n I \otimes I \right)^{-1}  \frac{1}{n} \sum_{i=1}^n \Psi_i^\ast(\theta),
\end{align*}
where $\widehat{\Omega}_\theta = \frac{1}{n}\sum_{i=1}^n \Psi_i^\ast(\theta) \Psi_i(\theta)$ denotes the empirical covariance operator of the moment functional. For any $i=1,\ldots,n$, the operator $\Psi_i^\ast(\theta) \Psi_i(\theta)$ has at most rank $1$ and thus as the sum of $n$ such operators $\widehat{\Omega}_\theta$ can have most rank $n$. As $\widehat{\Omega}_\theta$ is an infinite dimensional linear operator $\mathcal{H} \rightarrow \mathcal{H}$ it thus must be rank deficient for any finite $n$ and therefore singular and non-invertible.
This highlights the fact that the regularization parameter in the FGEL framework is not merely an artefact of the restoration of the strong duality between the primal and dual GEL problems, but a fundamental requirement for any definition of a functional/continuum GEL extension.
Note that by the uniform weak law of large number and the continuous mapping theorem we have for $\Psi(\theta)$ continuous in $\theta$ and $\hat{\theta} \overset{p}{\rightarrow} \theta_0$ that $\widehat{\Omega}(\hat{\theta}) \overset{p}{\rightarrow} \Omega_0 = E[\Psi(X,Z;\theta_0)^\ast \Psi(X,Z;\theta_0)]$. 
In Theorem~\ref{th:1:consistency-cmr} we show that for the conditional case under the mild assumption that $V(Z) := E[\psi(X;\theta_0) \psi(X;\theta_0)^T | Z]$ is non-singular with probability $1$ together with assumptions on the instrument functions class $\mathcal{H}$ it follows that $\Omega_0$ is non-singular and thus in the limit $n \rightarrow \infty$ the first order conditions for $h$ remain well posed as $\lambda_n \rightarrow 0$.
The general formulation \eqref{eq:1:regularized-fgel} allows us to employ a wide range of function classes $\mathcal{H}$ and generally for finite samples, the choice of $\mathcal{H}$ will influence the obtained estimator. 
Building on recent developments in machine learning, we can represent $h$ by a flexible deep neural network \citep{deepiv,lewis2018adversarial} or a random forest model \citep{Athey19:GRF}, for example. In this work, we mainly focus our discussion on instrument functions from reproducing kernel Hilbert spaces for their favorable theoretical properties but also consider neural network function classes. 
\begin{remark}
The FGEL framework admits an interesting relation to \emph{distributionally robust optimization} and as such can be used for (distributionally) robust learning. Refer to Section~\ref{appendix:dro} of the appendix for a more detailed account of this connection.
\end{remark}
\subsection{Asymptotic Properties}
In this section, we establish asymptotic properties of our estimator given in \eqref{eq:1:regularized-fgel}.
The proofs generalize the ones of \citet{newey04} for the GEL estimator with finite dimensional moment restrictions to our regularized problem with functional-valued moment restrictions. Let in the following $\mathcal{H}_1$ denote the unit ball in $\mathcal{H}$.
\begin{theorem}[Consistency]\label{th:1:consistency}
Assume that 
a)~$\theta_0 \in \Theta$ is the unique solution to $\| E[\Psi(X,Z;\theta)]\|_{\mathcal{H}^\ast} = 0 $;
b)~$\Theta$, $\mathcal{X}$, $\mathcal{Z}$ are compact;
c)~$\Psi(x,z;\theta)$ is continuous in $x,z$ and $\theta$ everywhere; \\
d)~$E[\left(\sup_{\theta \in \Theta} \| \Psi(X,Z;\theta) \|_{\mathcal{H}^\ast}\right)^\nu] < \infty$ for some $\nu > 2$;
e)~$\Omega_0 = E[\Psi(X,Z;\theta_0) \otimes \Psi(X,Z;\theta_0)]$ is non-singular; 
f)~$\lambda_n = O_p(n^{-\xi})$ with $0 < \xi < 1/2$; 
g)~$\phi$ is twice continuously differentiable in a neighborhood of zero and $\phi_1(0)\neq 0$, $\phi_2(0) < 0$; 
h)~the class of functions $\{\Psi(\cdot;\theta)(h) : \ \theta \in \Theta, \ h \in \mathcal{H}_1 \}$ is $P_0$-Donsker.
Let $\hat{\theta}$ denote the FGEL estimator for $\theta_0$, then $\hat{\theta} \overset{p}{\rightarrow} \theta_0$ and $\| E[\Psi(X,Z;\hat{\theta})] \|_{\mathcal{H}^\ast} = O_{p}(n^{-1/2})$.
If additionally 
i)~$\theta_0 \in \operatorname{int}(\Theta)$; 
j)~$\Psi(x,z;\theta)$ is continuously differentiable in a neighborhood $\bar{\Theta}$ of $\theta_0$ \\ and $E[\sup_{\theta \in \bar{\Theta}} \| \nabla_\theta \Psi(X,Z;\theta) \|_{\mathcal{H}^\ast}^2] < \infty$; 
and k)~$\Sigma_0 := \left\langle E[\nabla_\theta \Psi(X,Z;\theta_0)], E[\nabla_{\theta^T}\Psi(X,Z;\theta_0)]\right\rangle_{\mathcal{H}^\ast}$ is a non-singular matrix in $\mathbb{R}^{p\times p}$, we have $\|\hat{\theta} - \theta_0 \|_2^2 = O_p(n^{-1/2})$.
\end{theorem}
The following theorem shows that the limiting distributions of the variables follow a normal distribution $N$ with covariance matrix $\Sigma_\theta$ and Gaussian process $\mathcal{N}$ with kernel $\Sigma_h$ respectively. 
\begin{theorem}[Asymptotic normality]\label{th:1:asymptotic-normality}
Let the assumptions of Theorem~\ref{th:1:consistency} be satisfied 
and define the parameter-gradient of the moment functional as $\nabla_\theta \Psi_0 := E[\nabla_\theta \Psi(X,Z;\theta_0)]$.
Then,
\begin{align*}
    \sqrt{n} (\hat{\theta} - \theta_0)  \overset{d}{\rightarrow}  N(0, \Sigma_\theta),  \quad
    \sqrt{n} (\hat{h} - h)  \overset{d}{\rightarrow} \mathcal{N}(0, \Sigma_h),
\end{align*}
where 
    $\Sigma_\theta = ((\nabla_{\theta} \Psi_0) \Omega_0^{-1} (\nabla_{\theta^{T}} \Psi_0^\ast))^{-1}$ and 
    $\Sigma_h = \Omega_0^{-1} -   \Omega_0^{-1} (\nabla_{\theta^{T}} \Psi_0^\ast)\Sigma_\theta  (\nabla_{\theta} \Psi_0) \Omega_0^{-1}$.
\end{theorem}
\subsection{From Functional to Conditional Moment Restrictions}
Theorems~\ref{th:1:consistency} and \ref{th:1:asymptotic-normality} show that FGEL provides a $n^{-1/2}$-consistent and asymptotically normal estimator for \emph{functional} moment restrictions of the form $\| E[\Psi(X,Z;\theta_0)] \|_{\mathcal{H}^\ast} = 0$. While this is of independent interest, often functional moment restrictions are merely a way to express a set of corresponding conditional moment restrictions $E[\psi(X;\theta_0)|Z]= 0 \ P_Z\mathrm{-a.s.}$ using a sufficiently expressive function space $\mathcal{H}$ in the corresponding variational form \eqref{eq:1:unconditional}. The following theorems show how the results on the functional formulation translate to the case of conditional moment restrictions. These results were not contained in the conference version of this manuscript and are originally due to \citet{pmlr-v202-kremer23a}.

\begin{theorem}[Consistency]\label{th:1:consistency-cmr}
    Let $\mathcal{H} \subseteq L^2(\mathcal{Z}, \mathbb{R}^m, P_Z)$ be a Hilbert space of locally Lipschitz functions which is sufficiently rich such that equivalence between \eqref{eq:1:conditional} and \eqref{eq:1:unconditional} holds.
    Further assume that 
    a) $\theta_0 \in \Theta$ is the unique solution to $E[\psi(X;\theta)|Z] = 0 \ P_Z\text{-a.s.}$; 
    b) $\Theta \subset \mathbb{R}^p$ as well as $\mathcal{X} \times \mathcal{Z}$ are compact; 
    c) $\psi(x;\theta)$ is continuous in $x$ and $\theta$ everywhere; 
    d) $E[\sup_{\theta \in \Theta} \| \psi(X;\theta) \|^2_2|Z] < \infty$ w.p.1; 
    e)~$V_0(Z) := E[\psi(X;\theta_0) \psi(X;\theta_0) |Z]$ is non-singular w.p.1;
    and 
    f)~$\lambda_n = O_p(n^{-\xi})$ with $0 < \xi < 1/2$; and
    g)~$\phi$ is twice continuously differentiable in a neighborhood of zero and $\phi_1(0)\neq 0$, $\phi_2(0) < 0$ and
    h)~the function classes $\{\psi(\cdot;\theta) : \theta \in \Theta \}$ and $\mathcal{H}_1$ are $P_0$-Donsker.
Then for the FGEL estimator $\hat{\theta}$ we have $\hat{\theta} \overset{p}{\rightarrow} \theta_0$.

Assume additionally 
    i)~$\theta_0 \in \operatorname{int}(\Theta)$;
    j)~$\psi(x;\theta)$ is continuously differentiable in a neighborhood $\bar{\Theta}$ of $\theta_0$ \\ 
    and $E[\sup_{\theta \in \bar{\Theta}} \| \nabla_\theta \psi(X;\theta) \|^2 | Z] < \infty$ w.p.1; as well as
    k)~$\operatorname{rank}\left(E[\nabla_\theta \psi(X;\theta_0) |Z] \right) = p$ w.p.1.
    Then we have $\|\hat{\theta} - \theta_0 \| = O_{p}(n^{-1/2})$.
\end{theorem}

\begin{theorem}[Asymptotic Normality]\label{th:1:asymptotic-normality-cmr}
Let the assumptions of Theorem~\ref{th:1:consistency-cmr} be satisfied. 
Then, for the FGEL estimator $\hat{\theta}$ we have
\begin{align*}
    \sqrt{n} (\hat{\theta} - \theta_0)  \overset{d}{\rightarrow}  N(0, \Xi_0)
\end{align*}
where 
    $\Xi_0 = E\left[  E[\nabla_\theta \psi(X;\theta_0)|Z] \ V_0^{-1}(Z) \ E[\nabla_\theta \psi(X;\theta_0)|Z] \right]^{-1}$.
\end{theorem}
The asymptotic variance in Theorem~\ref{th:1:asymptotic-normality-cmr} agrees with the semi-parametric efficiency bound of \citet{CHAMBERLAIN1987305}. This implies that FGEL provides an efficient estimator for conditional moment restrictions.
\begin{corollary}[Efficiency]\label{efficiency-cmr}
    Let the assumptions of Theorem~\ref{th:1:consistency-cmr} be satisfied. Then the FGEL estimate $\hat{\theta}$ is an efficient estimator for $\theta_0$, i.e., it has the smallest asymptotic variance among all  estimators based solely on the conditional moment restrictions $E[\psi(X;\theta_0)|Z]= 0 \ P_Z\text{-a.s.}$.
\end{corollary}
In order to translate the results for the functional estimator to conditional moment restrictions by applying Theorems~\ref{th:1:consistency-cmr}-\ref{efficiency-cmr} one needs to choose a space of instrument functions $\mathcal{H}$ which fulfills the corresponding conditions. In the following we show that the reproducing kernel Hilbert space of certain kind of kernel fulfills these assumptions.
\subsection{Kernel FGEL}
The definition of our FGEL estimator contains a supremum over a function space $\mathcal{H}$. In order to address the conditional moment restriction problem, the function space must be expressive enough to exhibit an equivalent unconditional formulation. 
At the same time, optimization over function spaces is generally intractable and thus requires approximations. Selecting instrument functions from a reproducing kernel Hilbert space, one obtains a computationally efficient formulation involving finite dimensional parameters.
\paragraph{Reproducing kernel Hilbert spaces}
Let $\mathcal{X}$ be a non-empty set and $\mathcal{H}$ a Hilbert space of functions $f: \mathcal{X} \rightarrow \mathbb{R}$. Let $\langle \cdot, \cdot \rangle_{\mathcal{H}}$ and $\| \cdot \|_{\mathcal{H}}$ denote the inner product and norm on $\mathcal{H}$ respectively. Then $\mathcal{H}$ is called a reproducing kernel Hilbert space (RKHS) if there exists a symmetric function $k: \mathcal{X} \times \mathcal{X} \rightarrow \mathbb{R}$ such that $k(x,\cdot) \in \mathcal{H}$ for all $x \in \mathcal{X}$ and $\langle f, k(x,\cdot) \rangle_{\mathcal{H}} = f(x)$ for all $f \in \mathcal{H}$ and $x \in \mathcal{X}$. 
Every positive (semi-)definite kernel is the unique reproducing kernel of an RKHS. We call a reproducing kernel $k$ \emph{integrally strictly positive definite} (ISPD) if additionally for any $f \in \mathcal{H}$ with $0<\| f\|_2^2 < \infty$ we have $\int_{\mathcal{X}} f(x) k\left(x, x^{\prime}\right) f\left(x^{\prime}\right) \mathrm{d} x \mathrm{~d} x^{\prime}>0$. See, e.g., \citet{scholkopf2002learning} for a comprehensive introduction.
Let $\mathcal{H} = \bigoplus_{i=1}^m \mathcal{H}_i$ denote the direct sum of $m$ RKHS of universal kernels $k_i$ \citep{universalkernel06}.
The following theorem which is based on Theorem~3.2 of \citet{muandet2020kernel} shows that the RKHS corresponding to a universal ISPD kernel (e.g., Gaussian kernel) is expressive enough to represent the conditional moment restriction \eqref{eq:1:conditional} in terms of a continuum of unconditional restrictions.
\begin{theorem} \label{th:1:equivalence}
Let $\mathcal{H} = \bigoplus_{i=1}^m \mathcal{H}_i$ denote the direct sum of $m$ RKHS unit balls $\mathcal{H}_{i}$ corresponding to ISPD kernels $k_i$, $i=1,\ldots,m$. Let $P$ denote a distribution over random variables $X \in \mathcal{X}$ and $Z \in \mathcal{Z}$ with marginal distributions $P_X$ and $P_Z$.
Then
\begin{align}
    E_{P}[\psi(X;\theta)|Z] = 0 \ \  P_Z\textrm{-a.s.}, \label{eq:1:conditional-rkhs}
\end{align}
if and only if  
\begin{align}
    E_{P}[\psi(X;\theta)^{T}h(Z)] = 0 \ \ \forall h \in \mathcal{H}. \label{eq:1:unconditional-rkhs}
\end{align}
\end{theorem}
With this result at hand, we can apply Theorems~\ref{th:1:consistency-cmr}-\ref{efficiency-cmr} to show that FGEL combined with such an RKHS as instrument functions space, which we term Kernel FGEL, provides an efficient estimator for CMR problems as formalized in the following corollary.
\begin{corollary} \label{cor:consistency-rkhs}
    The RKHS of a universal ISPD kernel satisfies the assumptions of Theorems~\ref{th:1:consistency-cmr}-\ref{efficiency-cmr}. Thus, under Assumptions a)-j) of Theorem~\ref{th:1:consistency-cmr} Kernel FGEL provides a $\frac{1}{\sqrt{n}}$-consistent, asymptotically normal and semi-parametrically efficient estimator for conditional moment restriction problems.
\end{corollary}

Applying the representer theorem \citep{Schoelkopf01:Representer} to the supremum over the instrument functions $h$ in equation \eqref{eq:1:regularized-fgel} allows us to represent the RKHS function in terms of finite dimensional parameters $\alpha_r \in \mathbb{R}^{n}$, $r=1,\ldots,m$, and yields a finite dimensional and convex optimization problem as formalized by the following lemma.

\begin{lemma} \label{th:1:optimization-problem}
    Let $\mathcal{H} = \bigoplus_{i=1}^m \mathcal{H}_i$ be an RKHS
    corresponding to $m$ universal kernels $k_i$, $i=1,\ldots,m$. Let $K_r\in\mathbb{R}^{n\times n}$, $r=1,\ldots,m$ denote the kernel matrices and let $\alpha = \{ \alpha_r\}_{r=1}^m$ with $\alpha_r \in \mathbb{R}^n$. 
    Then the maximization over the instrument functions in the FGEL objective \eqref{eq:1:regularized-fgel} can be expressed as 
    \begin{align*}
        R_{\lambda_n}(\theta) := \max_{\alpha \in \widehat{A}_\theta} & \Bigg\{  \frac{1}{n} \sum_{i=1}^n \phi \left( v_i(\theta,\alpha) \right) - \frac{\lambda_n}{2} \sum_{r=1}^m \alpha_r^{T}K_r \alpha_r \Bigg\},
    \end{align*}
    with $v_i = \sum_{r=1}^m (\alpha_r^{T}K_r)_{i} \psi_r(x_i;\theta)$ and $\widehat{A}_\theta = \{ \alpha :  v_i \in \operatorname{dom}(\phi), 1\leq i \leq n \}$. The Kernel FGEL estimator is then defined as the solution of $\hat{\theta} = \argmin_{\theta \in \Theta}R_{\lambda_n}(\theta)$.
\end{lemma}
We provide details on the optimization algorithm in Section~\ref{appendix:optimization} of the appendix. 

\subsection{Neural FGEL}
As expressed by universal approximation theorems (e.g., \citep{yarotsky2017error}), neural networks can represent arbitrarily large function classes and have shown state-of-the-art performance on related tasks \citep{deepiv,lewis2018adversarial, bennett2020deep}. As such, they provide a particularly interesting choice of instrument function class.
Let $h_\omega: \mathcal{Z} \rightarrow \mathbb{R}^m$ denote a feed-forward neural network with parameters $\omega$. Then we can define the Neural FGEL estimator as a saddle point of
\begin{align*}
    G_{\lambda_n}(\theta, \omega) :=  \frac{1}{n} \sum_{i=1}^n \phi \left( \psi(x_i;\theta)^{T}h_\omega(z_i)\right) - \frac{\lambda_n}{2n} \sum_{i=1}^n \| h_\omega(z_i) \|_{\mathbb{R}^m}^2,
\end{align*}
where the regularization term penalizes the magnitude of the output as in \citet{Dikkala20:Minimax} and \citet{bennett2020variational}. We leave the theoretical analysis of the Neural FGEL estimator for future work.
\subsection{Other Instrument Function Classes}
FGEL estimators can be defined for arbitrary instrument function classes $\mathcal{H}$ under mild conditions: 
Let $P$ denote a reference measure over $X \in \mathcal{X}$, then we can place a class $\mathcal{H}$ of functions $f: \mathcal{X} \rightarrow \mathbb{R}^m$ into the Hilbert space of square-integrable functions $L^2(\mathcal{H}, P)$ as long as any $f \in \mathcal{H}$ is bounded on any set with non-zero measure, which is a realistic assumption for many model classes. 
The corresponding norm with respect to the empirical measure is then given via $\|h \|_\mathcal{H}^2 = \frac{1}{n}\sum_{i=1}^n \|h(z_i) \|_{\mathbb{R}^m}^2$.
If the underlying problem of interest is a conditional moment restriction (instead of a general functional moment restriction), $\mathcal{H}$ additionally must be expressive enough such that an equivalence between the conditional \eqref{eq:1:conditional} and unconditional \eqref{eq:1:unconditional} formulations holds.

\subsection{Choice of Divergence Function}\label{sec:1:variants}
In this section, we discuss various choices of divergences and establish connections to existing methods.
In the finite dimensional case, it is well known that for any quadratic discrepancy function the GEL estimator coincides with the continuous updating GMM (CUE) estimator \citep{newey04}. 
An interesting special choice of divergence function is given below.
\begin{proposition}\label{prop:1:equivalence}
    Choosing the GEL function as $\phi(v) = - (1 \pm \frac{v}{2})^2$ and rescaling the regularization parameter $\tilde{\lambda}_n = 2\lambda_n$, the FGEL estimator becomes equivalent to the solution of the optimization problem
    \begin{align*}
        \min_{\theta \in \Theta} \sup_{h \in \mathcal{H}} \Big\{ E_{\hat{P}_n}[\psi(X;\theta)^{T}h(X)] - \frac{1}{4} E_{\hat{P}_n}\left[\left(\psi(X;\theta)^{T}h(X)\right)^2\right] 
          - \frac{\tilde{\lambda}_n}{4} \|h \|_{\mathcal{H}}^2 \Big\}.
    \end{align*}
\end{proposition}
This resembles the objective of the VMM estimator of \citet{bennett2020variational} with the only difference that the covariance term contains the decision variable $\theta$ instead of a first-stage estimate $\tilde{\theta}$. %
In this sense, with this special choice of divergence function our FGEL estimator and the VMM estimator are related in the same way as the continuous updating estimator (CUE) \eqref{eq:1:cue} and the optimally weighted 2-step GMM estimator \eqref{eq:1:owgmm}.
With the kernel version of our FGEL estimator, we can carry out the optimization over $h \in \mathcal{H}$ in closed form and similarly obtain a continuous updating version of the Kernel VMM estimator.
A functional generalization of the original empirical likelihood estimator is retrieved by setting $\phi(v) = - \log(1-v)$. The empirical likelihood estimator has many desirable properties. 
It has been shown by \citet{newey04} that the ordinary EL estimator has the smallest higher order bias among the family of GEL estimators (including GMM). 
Further, \citet{corcoran98} shows that confidence intervals constructed from the EL-based profile likelihood admit a Bartlett correction which by a simple subtraction allows to reduce the coverage error from $O(n^{-1})$ to $O(n^{-2})$. This property of the EL framework is unique among the family of GEL estimators \citep{corcoran98}.
Using the GEL function corresponding to the Kullback-Leibler (KL) divergence $\phi(v) = - e^v$ one obtains a functional generalization of the exponential tilting estimator of \citet{kitamura} and \citet{imbens1998} which shows good empirical performance on many tasks \citep{imbens1998}. In contrast to the $\chi^2$-divergence, the KL-divergence enjoys great popularity as a distributional divergence measure in machine learning \citep{blei}. 
Therefore, a functional moment restriction estimator based on the KL-divergence instead of the dominating $\chi^2$-divergence (GMM) could be of particular interest.
\begin{figure*}[t]
    \centering
        \includegraphics[width=0.6\linewidth]{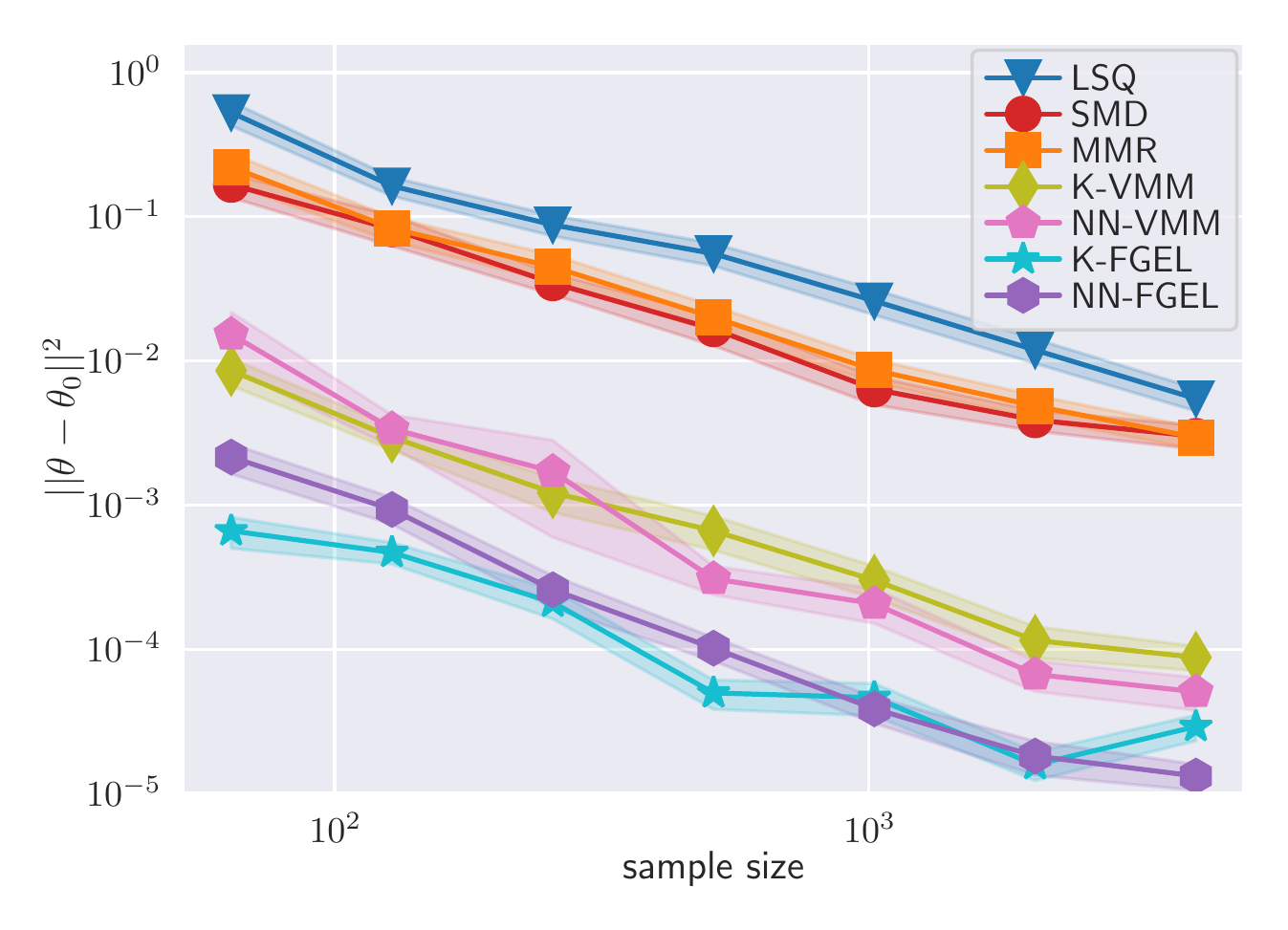}
        \caption{Estimation error over sample size for the heteroskedastic regression experiment. 
        Lines and shaded regions represent the MSE of the estimated parameters and the standard error averaged over $70$ runs respectively.}
        \label{fig:heteroscedastic}
\end{figure*}

\begin{figure}[t]
    \centering
    \includegraphics[width=0.6\linewidth]{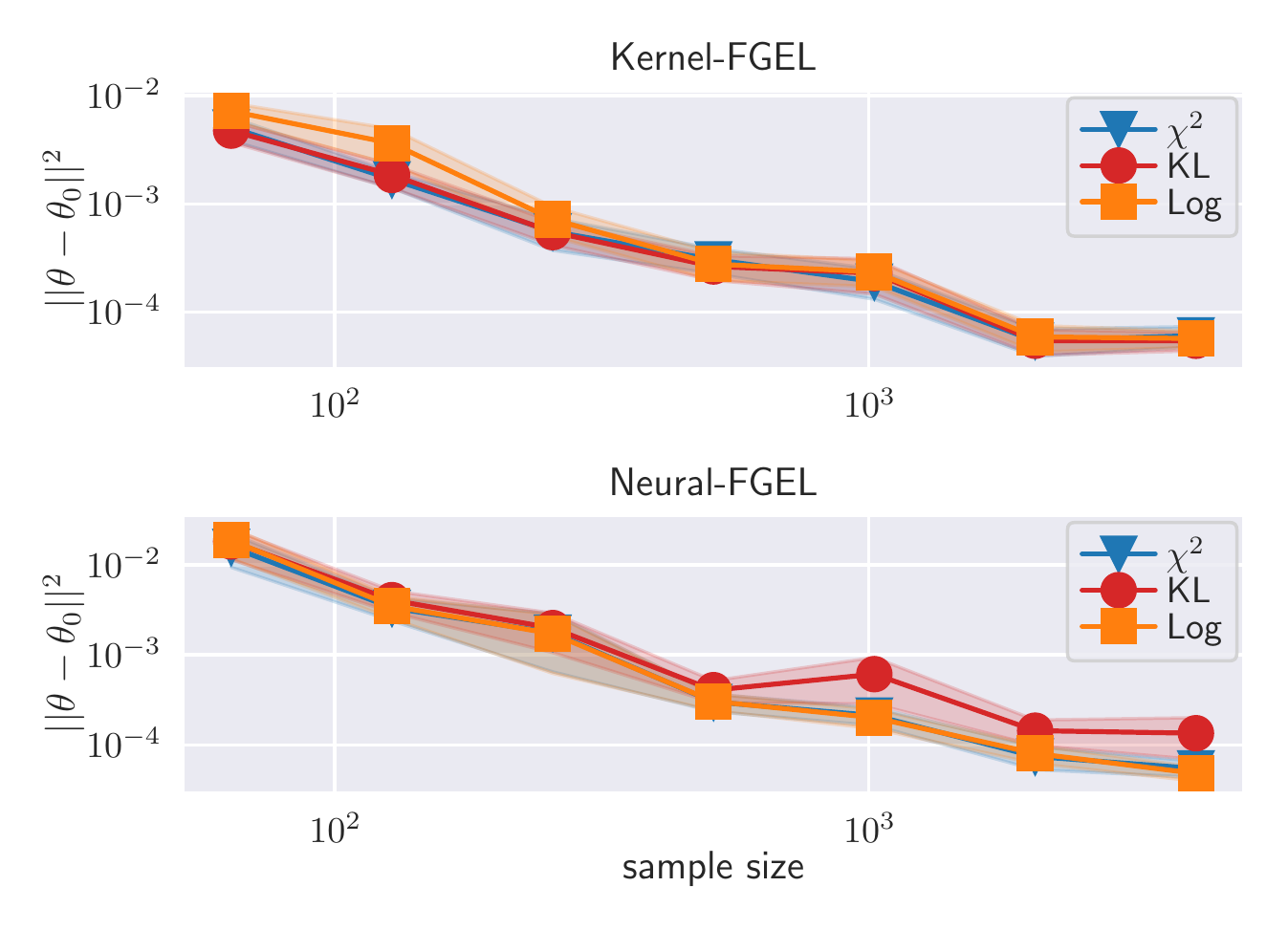}
    \caption{Comparison of different divergence functions. Lines and shaded regions represent the MSE of the estimated parameters and the standard error averaged over $70$ runs respectively.}
    \label{fig:heteroscedastic2}
\end{figure}
\section{Experiments}\label{sec:1:experiments}
For all experiments we use radial basis function kernels $k_i(x,x') = \exp(- \gamma \|x-x' \|^2)$, $i=1,\ldots,m$ and set the bandwidth parameter $\gamma$ via the common median heuristic \citep{scholkopf2002learning,garreau2018large}. If not stated otherwise, we tune the remaining hyperparameters of all methods by evaluating the MMR objective $\ell(\theta) = 1/n^2 \sum_{i,j=1}^n \psi(x_i;\theta)^TK_{ij} \psi(x_j;\theta)$ \citep{zhang2021maximum} on a validation set of the same size as the training set (refer to Section~\ref{appendix:hyperparams} of the appendix for details). We compare the performance of our kernel- and neural network-based methods with ordinary least-squares (LSQ), sieve minimum distance (SMD) \citep{ai2003efficient}, kernel maximum moment restrictions (MMR) \citep{zhang2021maximum} and the kernel- and neural network versions of the variational method of moments (K-VMM and NN-VMM) \citep{bennett2020variational,bennett2020deep} on two conditional moment restriction problems. 
Code for reproducing the experimental results is available at \url{https://github.com/HeinerKremer/Functional-GEL}.
\subsection{Linear Regression under Heteroskedastic Noise}
We define a simple data generating process for a one-dimensional estimation problem. Let $\theta=1.7 \in \mathbb{R}$ and
\begin{align*}
    y = x^T\theta + \varepsilon, \ \ \ \  x \sim \operatorname{Uniform}([-1.5, 1.5]), 
\end{align*}
where $\varepsilon$ describes heteroskedastic noise such that $\varepsilon|x \sim \mathcal{N}(0, \sigma = 5 x^2)$. We can formulate the regression task as the conditional moment restriction $E[Y - X^T\theta | X]=0 \ P_{X}\text{-a.s.}$. As $\varepsilon$ is a mean zero random variable, here, we can simply use the mean squared error on a validation set as an unbiased validation metric to tune the hyperparameters of all methods. 
Figure~\ref{fig:heteroscedastic} shows the mean-squared error (MSE) of the estimated parameters using different versions of FGEL and other state-of-the-art estimators for conditional moment restrictions in dependence on the sample size. Here we treat the choice of divergence as an additional hyperparameter.
We observe that both our methods yield the lowest parameter MSE and even slightly outperform the recently proposed state-of-the-art VMM estimator \citep{bennett2020variational}.
In Figure~\ref{fig:heteroscedastic2} we evaluate the effect of the divergence function. We observe that while the average performance is largely independent of the choice of divergence function, a comparison with the results shown in Figure~\ref{fig:heteroscedastic} reveals that for any fixed sample the different divergences yield estimators of different quality. Thus, treating the divergence as hyperparameter and choosing the estimator with the lowest validation loss, allows us to exceed the performance of the FGEL estimator with fixed divergence. 
As GMM-based methods implicitly build on the $\chi^2$-divergence, this highlights an advantage of our method which can leverage any $\varphi$-divergences. Note that for any \emph{fixed} divergence function the performance of our FGEL estimators are roughly on par with the corresponding VMM estimators.
\subsection{Instrumental Variable Regression}
We adopt a slightly modified version of the IV regression experiment of \citet{lewis2018adversarial}, which has also been used by \citet{bennett2020deep} and \citet{zhang2021maximum}.
Let the data generating process be given by
\begin{align*}
 &y = f_0(x) + e + \delta,   &x = z + e + \gamma, \\
 &z \sim \operatorname{Uniform}([-3, 3]),& \\
   & e \sim N(0,1),  & \gamma, \delta \sim N(0, 0.1),
\end{align*}
where $f_0$ is picked from the following simple functions 
\begin{align*}
    &\operatorname{sin:} f_0(x) = \sin(x), & \operatorname{abs:} f_0(x) = |x|, \\ &\operatorname{linear:} f_0(x) = x,  & \operatorname{step:} f_0(x) = I_{\{x \geq 0 \}}.
\end{align*}
We approximate $f_0$ by a shallow neural network $f_\theta(x)$ with 2 layers of $[20, 3]$ units and leaky ReLU activation functions and base the estimation on the conditional moment restrictions $E[Y-f_\theta(X)|Z] = 0 \ \  P_{Z}\text{-a.s.}$. As generally the true model is not contained in this model class, this provides a typical case of model misspecification and theoretical properties of the our method (and equally all baseline methods) for this setting have yet to be developed (see \citet{Dikkala20:Minimax} for recent progress in this direction).
We use training and validation sets of size $n=2000$ and evaluate the prediction error on a test set of $20000$ samples. The results are visualized in Table~\ref{tab:1:network-iv}.
We observe that with the exception of one task the FGEL and VMM estimators outperform all other baselines.
Compared to each other NN-FGEL seems to be preferable over NN-VMM but the kernel versions of both methods exhibit similar performance without showing a clear advantage of one over the other for this task.

\begin{table*}[t]
    \centering
    \caption{Prediction MSE for the instrumental variable task. Mean and standard deviation of the mean are computed over 50 random runs and multiplied by $10$ for ease of presentation.}
    \label{tab:1:network-iv}
    \resizebox{\textwidth}{!}{
    \begin{tabular}{lccccccc}
    \toprule
            & LSQ           & SMD           & MMR           & K-VMM         & NN-VMM        & K-FGEL        & NN-FGEL       \\
     \midrule
     abs    & $3.72\pm0.30$ & $2.97\pm0.97$ & $2.78\pm0.60$ & $0.43\pm0.15$ & $0.45\pm0.10$ & $\bf{0.17\pm0.01}$ & $0.23\pm0.06$ \\
     step   & $3.03\pm0.03$ & $0.37\pm0.04$ & $0.71\pm0.03$ & $\bf{0.31\pm0.01}$ & $0.41\pm0.01$ & $0.41\pm0.03$ & $0.34\pm0.01$ \\
     sin    & $3.28\pm0.04$ & $\bf{1.01\pm0.06}$ & $3.61\pm0.07$ & $1.55\pm0.12$ & $1.72\pm0.11$ & $1.97\pm0.16$ & $1.66\pm0.12$ \\
     linear & $2.76\pm0.06$ & $0.97\pm0.72$ & $1.98\pm0.38$ & $0.31\pm0.06$ & $0.34\pm0.05$ & $0.32\pm0.05$ & $\bf{0.20\pm0.03}$ \\
    \bottomrule
    \end{tabular}}
\end{table*}

Our experiments show that the FGEL estimator is a viable alternative to previously proposed continuum method of moments estimators for conditional moment restrictions and can surpass the previous state-of-the-art on some tasks. However, further empirical evidence needs to be collected to verify its predicted superior finite sample properties for infinitely many moment restrictions. We leave a comprehensive experimental evaluation to future work.
\section{Related Work} \label{sec:1:related-work}
Learning with conditional or infinite dimensional moment restrictions respectively has been an active field of research in econometrics and more recently in machine learning. 
In the former context, seminal work on extending the generalized method of moments to continua of moment restrictions has been carried out by \citet{Carrasco1,Carrasco2} by placing the constraints in an RKHS.
In the machine learning community, GMM-related estimators have been developed by casting the infinite dimensional moment restriction problem as a minimax game and representing the adversarial player by an RKHS function \citep{zhang2021maximum, bennett2020variational} or a flexible neural network \citep{deepiv,lewis2018adversarial,Dikkala20:Minimax,bennett2020deep}. 
While the neural network-based methods often achieve good performance in practice, they generally are computationally more expensive and lack the theoretical properties of traditional GMM estimators. 
In contrast, \citet{bennett2020variational}'s Kernel VMM estimator comes with strong theoretical guarantees but results from a 2-step procedure and thus depends on an initial parameter estimate.
As discussed in Section~\ref{sec:1:variants}, our framework contains a continuous updating version of VMM as a special case but allows for using alternative $\varphi$-divergence functions. 
As an alternative to GMM estimation, sieve-based methods \citep{newey2003instrumental,donald2003,ai2003efficient,chen2012estimation} address conditional moment restrictions by growing the number of unconditional restrictions with the sample size by manually selecting an increasing number of basis functions.
While these often come with desirable efficiency results, in practice they can be hard to tune and computationally demanding \citep{bennett2020variational}. 
Another line of work implicitly estimates optimal instrument functions via a kernel-smoothed localized empirical likelihood function \citep{tripathi2003testing,kitamura2004}. Their use of kernels is different from our approach as we do not smooth the profile divergence but use RKHS functions (and other function classes) as instrument functions.
Several works extended the generalized empirical likelihood framework to handle infinite dimensional moment restrictions and thus conditional moment restrictions \citep{donald2003,chausse2012generalized,carrasco2017regularized}.
The GEL estimator of \citet{chausse2012generalized} is based on approximately imposing a continuum of moment restrictions using a parameterized basis of functions and solving a regularized version of the GEL first order conditions. While it is theoretically closely related to our method, the regularization scheme and computational approach differs from ours.
Similarly, closely related to our method is the regularized GEL estimator of \citet{carrasco2017regularized}, which is defined via a set of optimality conditions and solved using a procedure motivated by the Three-Steps Euclidean Likelihood procedure of \citet{antoine2007efficient}. In contrast to these methods, our estimator is defined as a saddle point of an objective function and thus benefits from recent advances in mini-max optimization \citep{daskalakis2018training,lin2020gradient}. 
To the best of our knowledge, our work is the first to combine GEL estimation with modern machine learning and in particular kernel methods and neural networks.
\section{Conclusion}
Several long-established problems in machine learning can naturally be expressed as a risk minimization problem. On the other hand, emerging areas such as causal inference, algorithmic decision making, and robust learning often involve problems that are formulated as (potentially infinite) moment restrictions and require different algorithmic frameworks for estimation and inference. Recent works have advanced this development by combining classical techniques from econometrics such as generalized method of moments (GMM) with modern machine learning models such as deep neural networks and kernel machines. Likewise, our work contributes to this endeavour by equipping the more general generalized empirical likelihood (GEL) framework with such powerful models.
While the econometrics community enjoys the new class of algorithms, we believe the machine learning community will likewise benefit from new perspectives on causal inference and robust learning which will be explored in future works.
This paper laid the theoretical foundation of the functional GEL framework, but there remain open questions that impede real-world applications. Firstly, more efficient optimization procedures need to be developed that allow for large scale applications. Secondly, theoretical properties of the framework with specific function classes need to be explored. Lastly, the framework needs to be tested for the training of more complex models for real-world applications (e.g.\ robust learning). Our goal is to address some of these problems in future work.
\section*{Acknowledgements}
We thank Simon Buchholz and Yassine Nemmour for helpful discussions and feedback on an earlier version of the manuscript.
This work was supported by the German Federal Ministry of Education and Research (BMBF): T\"ubingen AI Center, FKZ: 01IS18039B.

\bibliographystyle{abbrvnat}
\bibliography{refs}

\appendix
\newpage
\section{Additional Information}
\subsection{Distributional Robustness of FGEL}\label{appendix:dro}
It is well-known that the profile divergence is a dual formulation to the distributionally robust optimization (DRO) formulation \citep{lamRecoveringBestStatistical2019,duchi2018statistics}.
In the context of this paper, one can show that $R_{\lambda_n}(\theta) \leq \rho$ if and only if
\begin{equation*}
    \lambda_n \geq  \inf_{P \in \mathcal{P}} 
	\|
	E_{P}[\Psi(X,Z;\theta)]
	\|_{\mathcal{H}^\ast}
    \;\ 
    \mathrm{s.t.} \, D_\varphi(P||\hat{P}_n) \leq \rho.
\end{equation*}
However, we do not simply rely on the divergence-ball centered at the empirical data distribution $\{P\vert \ D_\varphi(P||\hat{P}_n) \leq \rho\}$ (referred to as an ambiguity set in the DRO literature) for robustness. Since that robustness is often used to account for the statistical error due to finite samples. Instead, we are concerned with a second and stronger layer of robustness.
First, note that the quantity $E_{P}[\Psi(X,Z;\theta)]$ is used to approximate the conditional moment constraint in our original formulation \eqref{eq:1:123}.
Since the goal of FGEL is to satisfy the the conditional moment restrictions $E_{P}[\psi(X;\theta) \,|\, Z]=0$ almost everywhere in the domain, in robust optimization terms,
we are robustifying against the instrument $Z$. The instrument $Z$ can create much stronger distribution shifts in the data-generating process than the mere statistical fluctuation described by divergence-ball-based DRO works following \citet{ben-talRobustSolutionsOptimization2013} and \citet{duchi2018statistics}.
We leave an alternative DRO algorithm against such strong distribution shifts for future work.
From another perspective, our method can also be seen as enforcing independence between $Z$ and the moment restriction, e.g., for IV regression the residual $\psi(X;\theta)=Y - f_\theta(X)$. Intuitively, we want the residual $Y - f_\theta(X)$ to be small and invariant to transformations of the $Z$ variable (marginal shift). This kind of robust learning strategy has also been studied in works by \citet{greenfeld2020robust,rothenhausler2020anchor,heinzedeml2019conditional}.
\subsection{Computing the FGEL Estimator}\label{appendix:optimization}
Problem \eqref{eq:1:regularized-fgel} is generally a non-convex-convex min-max problem in the parameter $\theta$ and function $h$. Let $h=h_\alpha$ be described by a finite dimensional set of parameters $\alpha \in A$, which is the case, e.g., for neural network function classes or RKHS after using a representer theorem. Furthermore let the set of parameters $A$ be compact. If additionally the parameterization leaves the convexity of the inner problem intact (e.g., in the case of kernel-FGEL) we can use a simplified version of Danskin's theorem \citep{danskin} to compute gradients of $R_{\lambda_n}(\theta) := \sup_{h \in \widehat{\mathcal{H}}_{\theta}}G_{\lambda_{n}}(\theta, h)$ in a principled way.
\begin{lemma}[Danskin]
    Let $\hat{h}(\theta)$ denote the solution of the inner convex optimization over $h \in \mathcal{H}_\theta$ such that $\hat{h}(\theta)=\argmax_{h \in \widehat{\mathcal{H}}_\theta} G_{\lambda_n}(\theta, h)$. Then the gradient of the profile divergence $R_{\lambda_n}(\theta)$ with respect to the parameters $\theta \in \Theta$ is given by
    \begin{equation*}
        \nabla R_{\lambda_n}(\theta) = \nabla G_{\lambda_n}(\theta, \hat{h}(\theta)).
    \end{equation*}
\end{lemma}
Therefore, we can adopt a gradient-based strategy for the outer optimization problem over $\theta$ using in each step the gradient estimate obtained from the solution of the inner maximization over $h$. Depending on the GEL function $\phi$ the optimization of both, the outer and inner problem can then be solved efficiently with an off-the-shelf solver e.g.\ using LBFGS (cf. Algorithm~\ref{alg:kernel}).
\begin{algorithm}[t]
   \caption{Kernel-FGEL}
   \label{alg:kernel}
\begin{algorithmic}
   \STATE {\bfseries Input:} data $(x_i,y_i,z_i)$, hyperparameter $\lambda$
  \WHILE{not converged}
       \WHILE{not converged}
        \STATE $\alpha \gets \operatorname{LBFGS}(G_\lambda(\theta,h_\alpha))$
       \ENDWHILE
       \STATE $\theta \gets \operatorname{LBFGS}(G_\lambda(\theta, h_\alpha))$
  \ENDWHILE
  \STATE {\bfseries Output:} Parameter estimate $\theta$
\end{algorithmic}
\end{algorithm}
For the case of a neural network instrument function classes, we build on the recent progress in mini-max optimization and employ the optimistic Adam optimizer \citep{daskalakis2018training} which has been developed to solve similar saddle point problems for training generative adversarial networks \citep{goodfellow2014generative} (cf. Algorithm~\ref{alg:neural}). Implementations of both approaches are available under \url{https://github.com/HeinerKremer/Functional-GEL}.
\begin{algorithm}[t]
   \caption{Neural-FGEL}
   \label{alg:neural}
\begin{algorithmic}
   \STATE {\bfseries Input:} data $(x_i,y_i,z_i)$, hyperparameter $\lambda$
  \WHILE{not converged}
       \STATE $\alpha \gets \operatorname{OAdam}(G_\lambda(\theta,h_\alpha))$
       \STATE $\theta \gets \operatorname{OAdam}(G_\lambda(\theta, h_\alpha))$
  \ENDWHILE
  \STATE {\bfseries Output:} Parameter estimate $\theta$
\end{algorithmic}
\end{algorithm}
\subsection{Hyperparameter selection}\label{appendix:hyperparams}
Tuning the hyperparameter of our method, i.e., the regularization parameter $\lambda_n$ (and, e.g., learning rates) requires a data-driven performance measure of the obtained model parameters. We know that for the true distribution $P_0$ and true parameter $\theta_0$ we obtain $\| E_{P_0}[\Psi(X,Z;\theta_0)]\|_{\mathcal{H}^\ast}^2 =0$. Let $\beta$ denote the set of hyperparameters and $\hat{\theta}(\beta)$ the corresponding solution to \eqref{eq:1:regularized-fgel}. Then we can define a performance measure of the solution candidate $\hat{\theta}(\beta)$ as $\ell(\beta) = \| E_{P_0}[\Psi(X,Z;\hat{\theta}(\beta))]\|_{\mathcal{H}^\ast}^2$.
As we do not have access to the true distribution $P_0$ we can define a natural surrogate loss $\hat{\ell}$ using a validation set with empirical distribution $\hat{P}_\text{val}$ as
\begin{align}
    \hat{\ell}(\beta) = \| E_{\hat{P}_{\text{val}}}[ \Psi(X,Z;\hat{\theta}(\beta)) ]\|_{\mathcal{H}^\ast}^2 \label{eq:1:surrogate}
\end{align}
Choosing $\mathcal{H}$ as an RKHS, this can be expressed as the kernel maximum of moment restriction objective of \citet{muandet2020kernel} and \cite{zhang2021maximum} evaluated on the validation data as shown by the following lemma.
\begin{lemma} \label{lemma:1:surrogate}
    Let $\{x_i, z_i\}_{i=1}^n$ denote the validation data and define $\pmb{\psi}_j(\pmb{x};\theta) = \operatorname{vec}(\{ \psi_j(x_i;\theta) \}_{i=1}^n)$. Let $K_j$ denote the kernel Gram matrix with entries $\left(K_j\right)_{pq}=k_j(z_p,z_q)$, $p,q=1,\ldots,n$, $j = 1, \ldots, m$.
    Then we can express \eqref{eq:1:surrogate} as 
    \begin{align*}
    \hat{\ell}(\beta) = \frac{1}{n^2} \sum_{j=1}^m \pmb{\psi}_j(\pmb{x};\theta)^T K_j \pmb{\psi}_j(\pmb{x};\theta).
    \end{align*}
\end{lemma}
Here we assume that possible hyperparameters of the kernel are already set via commonly employed heuristics like the median heuristic \citep{scholkopf2002learning,garreau2018large} for the kernel bandwidth and only tune the remaining parameters of our method.
\section{Proofs}\label{appendix:proofs}
\subsection{Preliminaries}
For ease of notation we define some expressions first. 
Define $\Psi_i(\theta) := \Psi(x_i,z_i;\theta)$ and denote $\phi_i(v) = \frac{d^i}{(d v)^i} \phi(v)$ and $\phi_i = \phi_i(0)$. Without loss of generality we assume that $\phi_1(0) = \phi_2(0) = -1$, as any $\phi$ with $\phi_1 \neq 0$ and $\phi_2 < 0$ can be rescaled to achieve this (see \citet{newey04}).
Define the empirical objective as $\widehat{G}_{\lambda_n}(\theta, h) = \sum_{i=1}^n \phi(\Psi(x_i,z_i;\theta)(h)) - \lambda \| h\|_\mathcal{H}^2$ and the empirical constraint set as $\widehat{\mathcal{H}}_n(\theta) = \{h \in \mathcal{H}:  \Psi(x_i,z_i;\theta)(h) \in \operatorname{dom}(\phi) \ \ \forall (x_i,z_i), \ i=1,\ldots,n \}$. Throughout the proofs we will make use of functional derivatives and a functional version of Taylor's theorem with Lagrange remainder, which we define and state next, respectively.
\begin{definition}[Functional Derivative] \label{def:func-dev}
    Let $\mathcal{H}$ be a vector space of functions. For a functional $G: \mathcal{H} \rightarrow \mathbb{R}$ and a pair of functions $h, \tilde{h} \in \mathcal{H}$, we define the derivative operator $D_{h} G(h)[\tilde{h}]=\left.\frac{d}{d t} G(h+t \tilde{h})\right|_{t=0}$. Likewise, we define 
    \begin{align*}
    D_{h}^{k} G(h)\left[h_{1}, \ldots, h_{k}\right] 
    = \left.\frac{\partial^{k}}{\partial t_{1} \ldots \partial t_{k}} G\left(h+t_{1} h_{1}+\ldots+t_{k} h_{k}\right)\right|_{t_{1}=\cdots=t_{k}=0}. 
    \end{align*}
    Similarly, when considering a function of a vector-valued parameter, $G: \Theta \subseteq \mathbb{R}^p \rightarrow \mathbb{R}$, we denote the $k$-th standard directional derivative at $\theta \in \Theta$ as $D^k_\theta G(\theta)(\theta_1, \ldots, \theta_k)$. Alternatively we use $\nabla_\theta G(\theta)$ to denote a row vector in $\mathbb{R}^{p}$ (the gradient) and $\nabla_{\theta^T} G(\theta)$ to denote the corresponding column vector in the dual space.
\end{definition} 
\begin{proposition}[Taylor's theorem] \label{prop:1:taylor}
    Let $G: \mathcal{H} \rightarrow \mathbb{R}$, where $\mathcal{H}$ is a vector space of functions. For any $h, h^{\prime} \in \mathcal{H}$, if $t \mapsto G\left(t  h + (1-t)  h^{\prime}\right)$ is $(k+1)$-times differentiable over an open interval containing $[0,1]$, then there exists $\bar{h} \in \operatorname{conv}\left(\left\{h, h^{\prime}\right\}\right)$ such that
    \begin{align*}
        G\left(h^{\prime}\right) = & \  G(h) + \sum_{i=1}^{k} \frac{1}{i !} D_{h}^{i} G(h)[\underbrace{h^{\prime}-h, \ldots, h^{\prime}-h}_{i \text { times }}] \\ 
        &+ \frac{1}{(k+1) !} D_{h}^{k+1} G(\bar{h})[\underbrace{h^{\prime}-h, \ldots, h^{\prime}-h}_{k+1 \text { times }}] . \nonumber
    \end{align*}
    Equally, using the notation of Definition~\ref{def:func-dev} the same result holds for functions of vector-valued parameters $G: \Theta \subseteq \mathbb{R}^p \rightarrow \mathbb{R}$.
\end{proposition}
Our duality result builds on Theorem~3.1 of \citet{borwein1993failure}. For completeness, we will state it here adapted to our notation. Note that while the theorem is already closely related to our result, a direct application of the theorem to our case is impeded as we additionally need to take into account the normalization constraint for $p$, i.e., $\sum_{i=1}^n p_i=1$.
\begin{proposition}[Borwein's theorem] \label{th:1:borwein}
For the problem
\begin{align}
    P = \inf_{p \in \mathbb{R}^n}  \sum_{i=1}^n \frac{1}{n} f(np_i) \quad  \mathrm{s.t.} \quad \|\frac{1}{n}\sum_{i=1}^n p_i \Psi(x_i,z_i;\theta)\|_{\mathcal{H}^\ast} =0 \label{eq:1:borwein0} 
\end{align}
where we assume the infimum is attained (when finite), consider for $\lambda > 0$ the relaxed problem
\begin{align}
    P_\lambda = \min_{p \in \mathbb{R}^n}  \sum_{i=1}^n \frac{1}{n} f(np_i) \quad
    \mathrm{s.t.} \quad \|\frac{1}{n}\sum_{i=1}^n p_i \Psi(x_i,z_i;\theta)\|_{\mathcal{H}^\ast} \leq \lambda. \label{eq:1:borwein1}
\end{align}
Then the value $P_\lambda$ equals the value of the dual program
\begin{align}
    D_{\lambda} = \max_{h \in \mathcal{H}} - \frac{1}{n} \sum_{i=1}^n f^\ast(\Psi(x_i,z_i;\theta)(h)) - \lambda \|h \|_\mathcal{H}, \label{eq:1:borwein2}
\end{align}
and the unique optimal solution of \eqref{eq:1:borwein1} is given by 
\begin{align*}
(p_\lambda)_i = \left( \frac{d}{dv} f^\ast \right)\left(\Psi(x_i, z_i;\theta)(\hat{h})\right), \quad i=1,\ldots,n,
\end{align*} where $\hat{h}$ is any solution of \eqref{eq:1:borwein2}. Moreover, as $\lambda \rightarrow 0$, $p_\lambda$ converges in mean to the unique solution of \eqref{eq:1:borwein0} and $P_\lambda \rightarrow P$.
\end{proposition}
For the proofs of the asymptotic properties of our estimator, we need the following results.
\begin{lemma}[Corollary~9.31, \citet{kosorok2008introduction}] \label{lemma:1:donsker}
    Let $\mathcal{F}$ and $\mathcal{G}$ be Donsker classes of functions. Then $\mathcal{F} + \mathcal{G}$ is Donsker. Further if additionally $\mathcal{F}$ and $\mathcal{G}$ are uniformly bounded, then $\mathcal{F} \cdot \mathcal{G}$ is Donsker.
\end{lemma}

\begin{lemma}[Lemma~18, \citet{bennett2020variational}] \label{lemma:1:bennett}
Suppose that $\mathcal{G}$ is a class of functions of the form $g : \Xi \rightarrow \mathbb{R}$, and that $\mathcal{G}$ is $P$-Donsker in the sense of \citet{kosorok2008introduction}. Then we have
\begin{align*}
    \sup_{g \in \mathcal{G}} E_{\ppn}[g(\xi)] - E[g(\xi)] = O_p(n^{-1/2}).
\end{align*}
\end{lemma}

\paragraph{Proof of Theorem~\ref{th:1:duality}}
\begin{proof}
The proof follows almost directly from application of Proposition~\ref{th:1:borwein} by taking into account the additional constraint $\sum_{i=1}^n p_i = 1$.
The dual problem can be derived by introducing Lagrange parameters $\nu > 0$ and $\mu \in \mathbb{R}$ and defining the Lagrangian
\begin{align*}
    L(\theta, p, \mu, \nu) = \sum_{i=1}^n \frac{1}{n} f(np_i) - \mu \left(\sum_{i=1}^n p_i - 1 \right)  + \nu \left( \|\sum_{i=1}^n p_i \Psi(x_i,z_i;\theta) \|_{\mathcal{H}^\ast} - \lambda \right).
\end{align*}
Using the definition of the dual norm and the fact that trivially $\lambda = \max_{\| h\| = 1} \| h \| \lambda $, we have
\begin{align*}
     L= \sum_{i=1}^n \frac{1}{n} f(np_i) - \mu \left(\sum_{i=1}^n p_i - 1 \right) + \sup_{\|\tilde{h}\|_{\mathcal{H}}= 1} \left( \sum_{i=1}^n  \langle \nu \tilde{h}, p_i \Psi(x_i,z_i;\theta) \rangle - \|\nu \tilde{h}\|_\mathcal{H} \lambda \right).
\end{align*}
By defining new dual Lagrange parameters $h = \nu \tilde{h} \in \mathcal{H}$, we thus obtain
\begin{align*}
     L(\theta, p, \mu, h) = \sum_{i=1}^n \frac{1}{n} f(np_i) - \mu \left(\sum_{i=1}^n p_i - 1 \right) +  \sum_{i=1}^n \langle h, p_i \Psi(x_i,z_i;\theta) \rangle - \|h\|_\mathcal{H} \lambda.
\end{align*}
Now, redefining $p_i \rightarrow n p_i$ and optimizing the Lagrangian with respect to p we get
\begin{align*}
    & \min_{p} \Big\{ \mu - \frac{1}{n} \sum_{i=1}^n \Big[ (\mu - \Psi(x_i,z_i;\theta)(h)) p_i - f(p_i) \Big]  - \lambda \|h \|_\mathcal{H} \Big\} \\
    =& \ \mu - \frac{1}{n} \sum_{i=1}^n \max_{p_i} \Big\{  (\mu - \Psi(x_i,z_i;\theta)(h)) p_i - f(p_i) \Big\}  - \lambda \|h \|_\mathcal{H} \\
    =& \ \mu - \frac{1}{n} \sum_{i=1}^n f^{\ast}(\mu - \Psi(x_i,z_i;\theta)(h)) - \lambda \|h \|_\mathcal{H},
\end{align*}
where we used the definition of the Legrendre-Fenchel (convex) conjugate function $f^\ast(v) = \sup_{x} \langle v, x \rangle - f(x)$. As for any $h \in \mathcal{H}$, $-h \in \mathcal{H}$, we can redefine $h \rightarrow -h$ and finally obtain the result.
Finally from Proposition~\ref{th:1:borwein} it follows that strong duality holds and the unique minimizer of the primal problem is given by
$$p_{i} = \left(\frac{d}{dv} f^\ast\right) \left(\Psi(x_i,z_i;\theta)(\hat{h}) + \hat{\mu}\right), \quad i=1,\ldots,n,$$ where $\hat{h}$, $\hat{\mu}$ are any solutions of the dual problem.
\end{proof}
\subsection{Asymptotic Properties of FGEL}
\subsubsection{Proof of Theorem~\ref{th:1:consistency} (Consistency)}
For the proof of Lemma \ref{lemma:1:first} we will need the following result whose proof closely follows a similar result for vector-valued moment restrictions of \citet{owen90} and \citet{kitamura2004} (Lemma~D.2):
\begin{lemma}\label{lemma:1:d2}
Let X be a random variable taking values in $\mathcal{X} \subseteq \mathbb{R}^r$. For a functional $\Psi: \mathcal{X} \times \Theta \times \mathcal{H} \rightarrow \mathbb{R}$ with $E\left[ \left(\sup _{\theta \in \Theta}\|\Psi(X; \theta)\|_{\mathcal{H}^\ast}\right)^{m}\right]<\infty$, it follows that $\max _{1 \leq j \leq n} \sup _{\theta \in \Theta}\left\|\Psi\left(x_{j}; \theta\right)\right\|_{\mathcal{H}^\ast}=o\left(n^{1 / m}\right)$ with probability 1.
\end{lemma}
\begin{proof}
For ease of notation define the random variable $Y := \sup _{\theta \in \Theta}\left\|\Psi\left(X; \theta\right) \right\|_{\mathcal{H}^\ast}$ and let for $i \in \mathbb{N}$, $Y_i$ denote independent copies of $Y$. Then as $E[Y^m] < \infty$, we must have that $\sum_{i=1}^\infty P(Y_i^m > n) < \infty$ or equivalently $\sum_{i=1}^\infty P(Y_i > n^{1/m}) < \infty$.
Hence by the Borel-Cantelli Lemma the event $Y_i > n^{1/m}$ happens only finitely often with probability $1$ which likewise implies $Z_n := \max_{1\leq i \leq n} Y_i > n^{1/m}$ happens only finitely often with probability $1$. By the same argument the event $Z_n > \epsilon n^{1/m}$ happens only finitely often for any $\epsilon > 0$ and thus
$$
\limsup Z_n / n^{1/m} \leq \epsilon
$$
with probability $1$ and thus $Z_n = o(n^{1/m})$ with probability $1$.
\end{proof}
The following Lemma shows that if we constrain the space of the dual parameter to a ball of radius $\zeta$ with $1/\nu < \zeta < 1/2 - \xi$, the largest value the empirical moment functional evaluated on the dual parameter can take converges to zero in probability. Furthermore any such ball is contained in the empirical constraint set $\widehat{\mathcal{H}}(\theta) = \{h \in \mathcal{H}: \   \psi(x_i;\theta)^{T}h(z_i) \in \operatorname{dom}(\phi), \  1\leq i \leq n \}$.
\begin{lemma}\label{lemma:1:first}
Let the assumptions of Theorem~\ref{th:1:consistency} be satisfied, then for any $\zeta$ with $1/\nu < \zeta < 1/2$ define $\mathcal{H}_n = \{h \in \mathcal{H}: \| h \|_{\mathcal{H}} \leq n^{-\zeta} \}$. Then $\sup_{\theta \in \Theta, h \in \mathcal{H}_n, 1\leq i \leq n} |\Psi(x_i,z_i;\theta)(h)| \overset{p}{\rightarrow} 0$ and w.p.a.1, $\mathcal{H}_n \subseteq \widehat{\mathcal{H}}(\theta)$ for all $\theta \in \Theta$.
\end{lemma}
\begin{proof}
Using the Cauchy-Schwarz inequality together with Lemma~\ref{lemma:1:d2} we have
\begin{align*}
    &\sup_{\theta \in \Theta, h \in \mathcal{H}_n, 1\leq i \leq n} |\Psi(x_i,z_i;\theta)(h)| \\
    &\leq \sup_{\theta \in \Theta, h \in \mathcal{H}_n, 1\leq i \leq n} \left( \| h \|_\mathcal{H} \cdot \| \Psi(x_i,z_i;\theta) \|_{\mathcal{H}^\ast}  \right) \\
    & \leq n^{-\zeta}  \sup_{\theta \in \Theta, 1\leq i \leq n}  \| \Psi(x_i,z_i;\theta) \|_{\mathcal{H}^\ast} \\
    & = O_p(n^{-\zeta + 1/\nu}) \overset{p}{\rightarrow} 0. 
\end{align*}
As $V = \operatorname{dom}(\phi)$ is an open interval containing zero it follows that $\Psi(x_i,z_i;\theta)(h) \in \operatorname{dom}(\phi)$ w.p.a.1 for all $\theta \in \Theta$ and $h \in \mathcal{H}_n$.
\end{proof}

\begin{lemma} \label{lemma:1:clt}
    Let the assumptions of Theorem~\ref{th:1:consistency} be satisfied and assume $\bar{\theta} \overset{p}{\rightarrow} \theta_0$ with $E[\|\psi(X;\bar{\theta}) - \psi(X;\theta_0) \|_\infty] = O_p(n^{-\rho})$ with $0 < \rho < 1/2$. Define the operators
    \begin{align*}
        \Omega(\theta) &= E[\Psi(X,Z;{\theta}) \otimes \Psi(X,Z;{\theta})]
        \\ 
        \widehat{\Omega}(\theta) &= E_{\ppn}[\Psi(X,Z;{\theta}) \otimes \Psi(X,Z;{\theta})].
    \end{align*}
    Then we have $\| \widehat{\Omega}(\bar{\theta}) - \Omega(\theta_0) \| = O_p(n^{-\rho})$.
\end{lemma}
\begin{proof}
    The proof follows the proof of Lemma~20 of \citet{bennett2020variational}.
    Using the triangle inequality we have,
    \begin{align*}
        \|\widehat{\Omega}(\bar{\theta}) - \Omega_0 \| \leq \| \widehat{\Omega}(\bar{\theta}) - \Omega(\bar{\theta}) \| + \| \Omega(\bar{\theta}) - \Omega(\theta_0) \|.
    \end{align*}
    For the first term we have 
    \begin{align*}
        \| \widehat{\Omega}(\bar{\theta}) - \Omega(\bar{\theta}) \| &= \sup_{h,h' \in \mathcal{H}_1} E_{\ppn}[h(Z)^T \psi(X;\bar{\theta}) \psi(X;\bar{\theta})^T h'(Z) ] - E[h(Z)^T \psi(X;\bar{\theta}) \psi(X;\bar{\theta})^T h'(Z)] \\
        &= \sup_{g \in \mathcal{G}^2} E_{\ppn}[g(X,Z)] - E[g(X,Z)],
    \end{align*}
    where 
    \begin{align*}
        \mathcal{G} &= \{g: g(x,z) = h(z)^T \psi(x;\bar{\theta}), \ h \in \mathcal{H}_1\}, \\
        \mathcal{G}^2 &= \{g: g(x,z) = g_1(x,z) g_2(x,z), \ g_1, g_2 \in \mathcal{G} \}.
    \end{align*}
    Now by Assumption~h), $\mathcal{G}$ is $P_0$-Donsker and uniformly bounded by continuity of $\Psi$ and compactness of its domain. Therefore $\mathcal{G}^2$ is $P_0$-Donsker by Lemma~\ref{lemma:1:donsker} which lets us employ Lemma~\ref{lemma:1:bennett} to conclude that $ \| \widehat{\Omega}(\bar{\theta}) - \Omega(\bar{\theta}) \| = O_p(n^{-1/2})$.

    For the second term we have
    \begin{align*}
        \| \Omega(\bar{\theta}) - \Omega(\theta_0) \| =& \sup_{h,h' \in \mathcal{H}} E[h(Z)^T \left(\psi(X;\bar{\theta}) \psi(X;\bar{\theta})^T - \psi(X;\theta_0) \psi(X;\theta_0)^T  \right) h'(Z)] \\
        \leq& \sum_{i,j=1}^m \sup_{h,h' \in \mathcal{H}} E[h_i(Z) h'_j(Z) \psi_i(X;\bar{\theta}) \left( \psi_j(X;\bar{\theta}) - \psi_j(X;\theta_0) \right)] \\ 
        & + \sum_{i,j=1}^m \sup_{h,h' \in \mathcal{H}} E[h_i(Z) h'_j(Z) \psi_j(X;\theta_0) \left( \psi_i(X;\bar{\theta}) - \psi_i(X;\theta_0) \right)]  \\
        \leq & 2 m^2 C_h^2 C_\psi E[\|\psi(X;\bar{\theta}) - \psi(X;\theta_0) \|_\infty] \\
        \leq & O_p(n^{-\rho}).
    \end{align*}
    Putting things together we get $\|\widehat{\Omega}(\bar{\theta}) - \Omega_0 \| = O_p(n^{-\rho})$. 
\end{proof}

\begin{lemma}\label{lemma:1:non-singular}
    Let the assumptions of Theorem~\ref{th:1:consistency} be satisfied and assume $\bar{\theta} \overset{p}{\rightarrow} \theta_0$ with $E[\|\psi(X;\bar{\theta}) - \psi(X;\theta_0) \|_\infty] = O_p(n^{-\rho})$.  
    Then for $\lambda_n = O_p(n^{-\xi})$ with $0< \xi< \rho$ and any $\Dot{h} \in \operatorname{conv}(\{0, \bar{h} \})$ with $\bar{h} \in \mathcal{H}_n$, as defined in Lemma~\ref{lemma:1:first}, the operator
    \begin{align*}
        \widehat{\Omega}_{\lambda_n}(\Dot{h}, \bar{\theta}) =  - \frac{1}{n}\sum_{i=1}^n \phi_2(\bar{\Psi}_i(\Dot{h})) (\bar{\Psi}_i \otimes \bar{\Psi}_i) + \lambda_n I \otimes I
    \end{align*}
    is non-singular w.p.a.1 with largest eigenvalue $C < \infty$.
\end{lemma}

\begin{proof}
    Define $\widehat{\Omega}(\bar{\theta}) = - \frac{1}{n} \sum_{i=1}^n \phi_2(\Psi(x_i, z_i;\bar{\theta})(\Dot{h})) \left( \Psi(x_i,z_i;\bar{\theta}) \otimes \Psi(x_i,z_i;\bar{\theta}) \right)$.
    As $\Dot{h}$ lies in between $0$ and $\bar{h}$ and $\bar{h} \in \mathcal{H}_n$, we have $\Dot{h} \in \mathcal{H}_n$ and thus by Lemma~\ref{lemma:1:first} we have $\sup_{\theta \in \Theta} \max_{1\leq i \leq n} |\phi_2(\Psi_i(\theta)(\Dot{h})) + 1| \overset{p}{\rightarrow}0$ and we have $\widehat{\Omega}(\bar{\theta}) = \frac{1}{n} \sum_{i=1}^n  \Psi(x_i,z_i;\bar{\theta}) \otimes \Psi(x_i,z_i;\bar{\theta})$ w.p.a.1.
    By Lemma~\ref{lemma:1:clt}, $\widehat{\Omega}(\bar{\theta})$ converges to the non-singular operator  $\Omega_0$ at rate $O_p(n^{-\rho})$ and as $\lambda_n = O_p(n^{-\xi})$ with $0< \xi< \rho$, we have that $\widehat{\Omega}_{\lambda_n}(\Dot{h}, \bar{\theta})$ is non-singular w.p.a.1.

    To bound the largest eigenvalue of $\widehat{\Omega}_{\lambda_n}: = \widehat{\Omega}_{\lambda_n}(\Dot{h}, \bar{\theta})$ consider any $h \in \mathcal{H}$ and 
    \begin{align*}
        \langle h,  \widehat{\Omega}_{\lambda_n} h \rangle_\mathcal{H} = \frac{1}{n} \sum_{i=1}^n \|\Psi(x_i,z_i;\bar{\theta})(h) \|^2 \leq \frac{1}{n} \sum_{i=1}^n \|\Psi(x_i,z_i;\bar{\theta}) \|_{\mathcal{H}^\ast}^2 \| h\|_\mathcal{H}^2 = C_\psi^2 \| h\|_\mathcal{H}^2,
    \end{align*}
    where we used that by assumption for any $\theta \in \Theta$ and $h \in \mathcal{H}_1$, $(x,z) \mapsto \psi(x;\theta)^T h(z)$ is a continuous function on a compact domain and therefore bounded. Thus the largest eigenvalue of $\widehat{\Omega}_{\lambda_n}$ must be bounded by some constant $C < \infty$. 
\end{proof}

The following lemma generalizes Lemma~A2 of \citet{newey04} to our regularized continuum formulation. It is different from a similar proof of \citet{chausse2012generalized} (Lemma~2), as our regularization procedure differs from theirs.

\begin{lemma}\label{lemma:1:second}
  Let the assumptions of Theorem~\ref{th:1:consistency} be satisfied. Additionally assume there is a consistent parameter estimate $\bar{\theta} \in \Theta$, $\bar{\theta} \overset{p}{\rightarrow} \theta_{0}$, with $\|E_{\hat{P}_n}[\Psi(X,Z;\bar{\theta})] \|_{\mathcal{H}^\ast} = O_{p}\left(n^{-1 / 2}\right)$ as well as $E[\|\psi(X;\bar{\theta}) - \psi(X;\theta_0) \|_\infty] = O_p(n^{-1/2})$. Further let $\lambda_n = O_p(n^{-\xi})$ where $0 < \xi < 1/2 - 1/\nu$. 
  Then $\bar{h}=$ $\argmax_{h \in \widehat{\mathcal{H}}(\bar{\theta})} \widehat{G}_{\lambda_n}(\bar{\theta}, h)$ exists w.p.a.1, $\| \bar{h}\|_\mathcal{H}=O_{p}\left(n^{-1/2}\right)$, and $\widehat{G}_{\lambda_n}(\bar{\theta},\bar{h}) \leq \phi(0) + O_p\left(n^{-1}\right)$.
\end{lemma}
\begin{proof}
    Let $\bar{\Psi}_i := \Psi_i(\bar{\theta}) := \Psi(x_i,z_i;\bar{\theta})$  and $\bar{\Psi} = \frac{1}{n} \sum_{i=1}^n \bar{\Psi}_i $.
    By Lemma~\ref{lemma:1:first} and twice continuous differentiability of $\phi(v)$ in a neighborhood of zero, $\widehat{G}_{\lambda_n}(\bar{\theta}, h)$ is twice continuously differentiable on $\mathcal{H}_n  = \{h: \| h \|_{\mathcal{H}} \leq n^{-\zeta} \}$ w.p.a.1. Then $\tilde{h} = \argmax_{h \in \mathcal{H}_n} \widehat{G}_{\lambda_n}(\bar{\theta}, h)$ exists w.p.a.1.
    Using Taylor's theorem (Proposition~\ref{prop:1:taylor}) we can expand the regularized GEL objective about $h = 0$ and obtain
    \begin{align*}
        \phi_0 &= \widehat{G}_{\lambda_n}(\bar{\theta}, 0) \\
        &\leq \widehat{G}_{\lambda_n}(\bar{\theta},\tilde{h}) \\
        &= \phi_0 - \bar{\Psi}(\tilde{h}) - \frac{1}{2} \underbrace{ \left[ - \frac{1}{n}\sum_{i=1}^n \phi_2(\bar{\Psi}_i(\Dot{h})) (\bar{\Psi}_i \otimes \bar{\Psi}_i) + \lambda_n I \otimes I \right]}_{=: \widehat{\Omega}_{\lambda_{n}}(\Dot{h}, \bar{\theta})}(\tilde{h}, \tilde{h})
    \end{align*}
    for some $\Dot{h}$ on the line between $0$ and $\tilde{h}$. 
    Now, by Lemma~\ref{lemma:1:non-singular} the regularized covariance operator $\widehat{\Omega}_{\lambda_{n}}(\Dot{h}, \bar{\theta})$ is positive definite with smallest eigenvalue $C > 0$ bounded away from zero w.p.a.1.
    Using this and subtracting $\phi_0$ on both sides yields
    \begin{align*}
        0 = \bar{\Psi}(\tilde{h}) -  \frac{1}{2} \langle \tilde{h}, \widehat{\Omega}_{\lambda_n}(\bar{\theta}) \tilde{h} \rangle_{\mathcal{H}}  \leq \|\bar{\Psi} \|_{\mathcal{H}^\ast} \|\tilde{h} \|_\mathcal{H} -  \frac{1}{2}  C \| \tilde{h} \|_\mathcal{H}^2,
    \end{align*}
    where in the second line we used the Cauchy-Schwarz inequality for the first term. This means we have $\frac{1}{2} C \| \tilde{h} \|_\mathcal{H} \leq  \|\bar{\Psi} \|_{\mathcal{H}^\ast}$ w.p.a.1.
    As by assumption $\| \bar{\Psi} \| = O_p(n^{-1/2})$ it follows that $\| \tilde{h} \| = O_p(n^{-1/2})$.
    Now, as $n^{-1/2} \leq n^{-\zeta}$ we have $\tilde{h} \in \operatorname{int}(\mathcal{H}_n)$ w.p.a.1.
    Then, as $\tilde{h}$ is a maximizer contained in the interior of the domain $\mathcal{H}_n$, it must correspond to a stationary point of $\widehat{G}_{\lambda_n}$, i.e., $\left(\partial \widehat{G}_{\lambda_n}/\partial h \right) (\bar{\theta}, \tilde{h}) = 0$. 
    However, from Lemma~\ref{lemma:1:first} it follows that w.p.a.1 $\tilde{h} \in \widehat{\mathcal{H}}(\bar{\theta})$ and as $\widehat{G}_{\lambda_n}(\bar{\theta}, h)$ is concave and $\widehat{\mathcal{H}}(\bar{\theta})$ is convex we must have $\widehat{G}_{\lambda_n}(\bar{\theta}, \tilde{h}) = \sup_{h \in \widehat{\mathcal{H}}(\bar{\theta})} \widehat{G}_{\lambda_n}(\bar{\theta}, h)$, which directly implies $\bar{h} = \tilde{h}$ and proves the first conclusion.
    The second conclusion follows directly as $\bar{h} \in \operatorname{int}(\mathcal{H}_n)$ and thus $\bar{h} = O_p(n^{-\zeta})$.
    Finally as $\|\bar{\Psi} \|_{\mathcal{H}^\ast} = O_p(n^{-1/2})$ by assumption, we have $\widehat{G}_{\lambda_n}(\bar{\theta}, \bar{h}) \leq \phi_0 + \| \bar{\Psi} \|_{\mathcal{H}^\ast} \|\bar{h} \|_\mathcal{H} - \frac{1}{2} C \|\bar{h} \|_\mathcal{H}^2  = \phi_0 + O_p\left(n^{-1}\right)$, which completes the proof.
\end{proof}
The following lemma uses Lemmas~\ref{lemma:1:first} and \ref{lemma:1:second} to show that the empirical moment functional $E_{\hat{P}_n}[\Psi(X,Z;\hat{\theta})]$ evaluated at the FGEL estimator $\hat{\theta}$ converges to zero in the dual norm. The proof closely follows the proof of Lemma~A3 of \citet{newey04}. 
\begin{lemma}\label{lemma:1:third}
Let the assumptions of Theorem~\ref{th:1:consistency} be satisfied and denote $\hat{\theta}$ the corresponding FGEL estimator $\hat{\theta}= \argmin_{\theta \in \Theta}\sup_{h \in \widehat{\mathcal{H}}(\theta)} \widehat{G}_{\lambda_n}({\theta}, h)$.
Then $\| E_{\hat{P}_n} [ \Psi(X,Z;\hat{\theta}) ] \|_{\mathcal{H}^\ast} = O_p\left(n^{-1/2}\right)$.
\end{lemma}
\begin{proof}
Define $\hat{\Psi}_{i} := \Psi_i(\hat{\theta}) :=\Psi(x_i,z_i;\hat{\theta})$ and $\hat{\Psi}= \frac{1}{n} \sum_{i=1}^n \hat{\Psi}_i$.
Let $\mu(\hat{\Psi})$ be the Riesz representer of $\hat{\Psi} \in \mathcal{H}^\ast$ in $\mathcal{H}$.
Further, let $1/\nu < \zeta < 1/2$ be defined as in Lemma~\ref{lemma:1:first} and consider $\tilde{h} = -n^{-\zeta} \mu(\hat{\Psi})/\| \mu(\hat{\Psi}) \|_{\mathcal{H}}$, which implies $\tilde{h} \in \mathcal{H}_n$ and therefore by Lemma~\ref{lemma:1:first} $\max_{1\leq i \leq n} \|\hat{\Psi}_i(\tilde{h}) \| \overset{p}{\rightarrow} 0$ and $\tilde{h} \in \widehat{\mathcal{H}}(\hat{\theta})$ w.p.a.1.
Using the same steps as in the proof of Lemma~\ref{lemma:1:second} we can Taylor expand the empirical FGEL objective about $h = 0$,
\begin{align*}
    & \widehat{G}_{\lambda_n}(\hat{\theta}, \tilde{h}) =  \phi(0) - \hat{\Psi}(\tilde{h}) -  \frac{1}{2} \langle \tilde{h}, \widehat{\Omega}_{\lambda_n}(\Dot{h},\hat{\theta}) \tilde{h} \rangle_\mathcal{H},
\end{align*}
for some $\Dot{h}$ on the line between 0 and $\tilde{h}$. 
Now by Lemma~\ref{lemma:1:non-singular} the largest eigenvalue of the regularized covariance operator $\widehat{\Omega}_{\lambda_n}(\Dot{h}, \hat{\theta})$ can be bounded by a constant $\tilde{C} > 0$ w.p.a.1.
Therefore, we have w.p.a.1,
\begin{align*}
    \widehat{G}_{\lambda_n}(\hat{\theta}, \tilde{h}) \geq  \phi(0) + n^{-\zeta} \| \hat{\Psi} \|_{\mathcal{H}^\ast} - C n^{-2\zeta},
\end{align*}
where for the second term we have used the definition of $\tilde{h}$ and the notation $C = \frac{1}{2} \tilde{C}$.
Consider $\bar{\theta} = \theta_0$ in Lemma~\ref{lemma:1:second}, for which the requirements are fulfilled as $\|E_{\hat{P}_n}[\Psi(X,Z;{\theta_0})] \|_{\mathcal{H}^\ast} = O_{p}\left(n^{-1 / 2}\right)$ by Lemma~\ref{lemma:1:bennett} using Assumption~h).
Moreover, being the solution to the mini-max problem, $(\hat{\theta}, \hat{h})$ correspond to a saddle point of the \emph{empirical} FGEL objective $\widehat{G}_{\lambda_n}$. Using this and Lemma~\ref{lemma:1:second} we have
\begin{align*}
    & \phi(0) + n^{-\zeta} \| \hat{\Psi} \|_{\mathcal{H}^\ast} - C n^{-2\zeta}
    \leq \widehat{G}_{\lambda_n}(\hat{\theta}, \tilde{h})
    \leq  \widehat{G}_{\lambda_n} (\hat{\theta}, \hat{h})
    \leq \sup_{h \in \widehat{\mathcal{H}}(\theta_0)} \widehat{G}_{\lambda_n}(\theta_0, h) 
    \leq \phi(0) + O_p(n^{-1}).
\end{align*}
Now, subtracting $\phi(0)$ on both sides and solving for $\|\hat{\Psi} \|_{\mathcal{H}^\ast}$, we obtain 
\begin{align}
    \|\hat{\Psi} \|_{\mathcal{H}^\ast} \leq O_p(n^{\zeta - 1}) + C n^{-\zeta} = O_p(n^{-\zeta}),
    \label{eq:1:19} 
\end{align}
which follows as $1 / \nu < \zeta < 1/2$ and therefore $\zeta - 1 <  - 1/2 < - \zeta$.
Consider any $\epsilon_n \rightarrow 0$ and let $\tilde{h} = - \epsilon_n \mu(\hat{\Psi})$. Then by \eqref{eq:1:19}, $\tilde{h} = o_p(n^{-\zeta})$ and therefore $\tilde{h} \in \mathcal{H}_n$ w.p.a.1. Then as previously we have
\begin{align*}
    \phi(0) - \hat{\Psi}(\tilde{h}) - C \| h \|_{\mathcal{H}}^2 
    =  \  \phi(0) + \epsilon_n \| \hat{\Psi} \|_{\mathcal{H}^\ast}^2 - C \epsilon_n^2 \| \hat{\Psi} \|_{\mathcal{H}^\ast}^2 
    \leq \  \phi(0) + O_p(n^{-1}).
\end{align*}
As $1-\epsilon_n C$ is bounded away from zero, for all $n$ large enough, we have $\epsilon_n  \|\hat{\Psi} \|_{\mathcal{H}^\ast}^2 = O_p(n^{-1})$. As this holds for all $\epsilon_n \rightarrow 0$, it follows that $ \|\hat{\Psi} \|_{\mathcal{H}^\ast} = O_p(n^{-1/2})$.
\end{proof}
\paragraph{Proof of Theorem~\ref{th:1:consistency}}
\begin{proof}
Define $\hat{\Psi}_i = \Psi(x_i,z_i;\hat{\theta})$ and $\hat{\Psi} = \frac{1}{n}\sum_{i=1}^n \hat{\Psi}_i$. By Assumption~h) and Lemma~\ref{lemma:1:bennett}, we have $\| \hat{\Psi}(\theta) - E [\Psi(X,Z;\theta)] \|_{\mathcal{H}^\ast} = O_p(n^{-1/2})$ for any $\theta \in \Theta$. From Lemma~\ref{lemma:1:third} we also have $\| \hat{\Psi} \|_{\mathcal{H}^\ast} = O_p(n^{-1/2})$ and thus using the triangle inequality we get
\begin{align*}
    \left\| E[\Psi(X,Z;\hat{\theta})] \right\|_{\mathcal{H}^\ast} &= \left\| E[\Psi(\hat{\theta})] -\hat{\Psi}  + \hat{\Psi} \right\|_{\mathcal{H}^\ast} \\
    & \leq \left\| E[\Psi(X,Z;\hat{\theta})] -\hat{\Psi}  \right\|_{\mathcal{H}^\ast} + \left\|  \hat{\Psi} \right\|_{\mathcal{H}^\ast} \\
    &= O_p(n^{-1/2}) \overset{p}{\rightarrow} 0.
\end{align*}
As by Assumption~a) $\theta_0$ is the unique parameter for which $ \| E[\Psi(X,Z;\theta)] \|_{\mathcal{H}^\ast} = 0$ it follows that $\hat{\theta} \overset{p}{\rightarrow} \theta_0$. 

Following the proof of Theorem~A.1 of~\citet{pmlr-v202-kremer23a} we can use this result to translate the convergence rate of the moment functional to a convergence rate of the FGEL estimator $\hat{\theta}$ using Assumptions~i)-k). The proof is identical to the one provided by \citet{pmlr-v202-kremer23a} and we state it here merely for completeness.

By the mean value theorem, there exists $\bar{\theta} \in \operatorname{conv}(\{\theta_0, \hat{\theta} \})$ such that 
\begin{align*}
    \Psi(X,Z;\hat{\theta}) = \Psi(X,Z;\theta_0) + (\hat{\theta} - \theta_0)^T \nabla_\theta \Psi(X,Z;\bar{\theta}).
\end{align*}
Using this we have
\begin{align*}
    \|E[\Psi(X,Z;\hat{\theta})] \|^2_{\mathcal{H}^\ast} & = \| \underbrace{E[\Psi(X,Z;\theta_0)]}_{=0} + (\hat{\theta} - \theta_0)^T E[ \nabla_\theta \Psi(X,Z;\bar{\theta})]\|^2_{\mathcal{H}^\ast} \\
    &= \left\langle (\hat{\theta} - \theta_0)^T E[ \nabla_\theta \Psi(X,Z;\bar{\theta})], (\hat{\theta} - \theta_0)^T E[ \nabla_\theta \Psi(X,Z;\bar{\theta})] \right\rangle_{\mathcal{H}^\ast} \\
    &= (\hat{\theta} - \theta_0)^T \underbrace{\left\langle  E[ \nabla_\theta \Psi(X,Z;\bar{\theta})], E[ \nabla_{\theta^T} \Psi(X,Z;\bar{\theta})] \right\rangle_{\mathcal{H}^\ast}}_{=: \Sigma(\bar{\theta})} (\hat{\theta} - \theta_0) \\
    &\geq \lambda_\text{min}\left(\Sigma(\bar{\theta})\right) \| \hat{\theta} - \theta_0 \|^2_2
\end{align*}
Now as $\hat{\theta} \overset{p}{\rightarrow} \theta_0$ and $\bar{\theta} \in \operatorname{conv}(\{\theta_0, \hat{\theta} \})$ we have $\bar{\theta} \overset{p}{\rightarrow} \theta_0$ and thus $\Sigma(\bar{\theta}) \overset{p}{\rightarrow} \Sigma(\theta_0) =: \Sigma_0$ by the continuous mapping theorem. By the non-negativity of the norm $\Sigma_0$ is positive-semi definite and non-singular by Assumption~k), thus the smallest eigenvalue of $\Sigma(\bar{\theta})$, $\lambda_\text{min}(\Sigma(\bar{\theta}))$, is positive and bounded away from zero w.p.a.1. Finally as $\|E[\Psi(X,Z;\hat{\theta})] \| = O_p(n^{-1/2})$ taking the square-root on both sides we have $\|\hat{\theta} - \theta_0 \| = O_p(n^{-1/2 })$.

\end{proof}
\subsubsection{Proof of Theorem~\ref{th:1:asymptotic-normality} (Asymptotic Normality)}

\begin{lemma}\label{lemma:1:invertible}
    Let $\Omega_0 = E[\Psi(X,Z;\theta_0)^\ast \Psi(X,Z;\theta_0)] \in \mathcal{H} \times \mathcal{H}$ and $\Sigma_0 =  \left\langle E[\nabla_\theta \Psi(X,Z;\theta_0)], E[  \nabla_{\theta^T} \Psi(X,Z;\theta_0)] \right\rangle_{\mathcal{H}^\ast} \in \mathbb{R}^{p\times p}$ be non-singular.
    Then the matrix
    \begin{align*}
    M = - \begin{pmatrix}
        0 & \nabla_\theta \Psi_0 \\
        \nabla_{\theta^T} \Psi_{0}^\ast & \Omega_0
    \end{pmatrix}.
    \end{align*}
    is invertible and its inverse is given by
    \begin{align*}
        M^{-1} = \begin{pmatrix}
            \Xi & - \Xi  \left( \nabla_\theta \Psi_0 \right) \Omega_0^{-1} \\ 
            -\Omega_0^{-1} \left( \nabla_{\theta^T} \Psi_0^\ast \right) \Xi & \Omega_0^{-1} + \Omega_{0}^{-1} \left( \nabla_{\theta^T} \Psi_0^\ast \right) \Xi \left( \nabla_\theta \Psi_0 \right) \Omega_0^{-1}
        \end{pmatrix}
    \end{align*}
    with $\Xi = \left(\left( \nabla_\theta \Psi_0 \right) \Omega_0^{-1} \left( \nabla_{\theta^T} \Psi_0^\ast \right) \right)^{-1} \in \mathbb{R}^{p\times p}$.
\end{lemma}
\begin{proof}
    In order to find the inverse of $M$ we resort to standard blockmatrix algebra. Note that the inverse of a matrix 
    \begin{align*}
        P = \begin{pmatrix}
            A & B \\ C & D
        \end{pmatrix}
    \end{align*}
    can be derived as
    \begin{align*}
        P^{-1} = \begin{pmatrix}
        (A - B D^{-1} C)^{-1} & - (A - B D^{-1} C)^{-1} B D^{-1} \\
        - D^{-1} C (A - B D^{-1} C)^{-1} & D^{-1} + D^{-1} C (A - B D^{-1} C)^{-1} B D^{-1}
        \end{pmatrix}
    \end{align*}
    whenever $D$ and its Schur complement in $P$, $P/D = A - B D^{-1}C$ are invertible (see e.g., \citet{bernstein2009matrix}).

    The Schur complement of $\Omega_0$ in $M$ is a matrix $\Gamma \in \mathbb{R}^{p\times p}$ defined as
    \begin{align*}
        \Gamma := M / \Omega_0 = - \left( \nabla_\theta \Psi_0 \right) \Omega_{0}^{-1} \left( \nabla_{\theta^T} \Psi_0^\ast  \right).
    \end{align*}
    Now, as $\Omega_0$ is a positive definite operator by assumption, its smallest eigenvalue is bounded away from zero. It immediately follows that the smallest eigenvalue of $\Omega^{-1}_0$, $\lambda_\mathrm{min}(\Omega_0^{-1}) > 0$ is bounded away from zero and thus we have
    for any $\theta \in \Theta$ with $\|\theta \|_2 >0$,
    \begin{align*}
        \theta^T \Gamma \theta &= - \theta^T \left( \nabla_\theta \Psi_0 \right) \Omega_0^{-1} \left( \nabla_{\theta_T} \Psi_0^\ast \right) \theta \\
        &\leq - \lambda_\mathrm{min}(\Omega_0^{-1}) \theta^T \underbrace{\left\langle E[\nabla_\theta \Psi(X,Z;\theta_0)], E[  \nabla_{\theta^T} \Psi(X,Z;\theta_0)] \right\rangle_{\mathcal{H}^\ast}}_{=\Sigma_0} \theta \\
        & \leq - \lambda_\mathrm{min}(\Omega_0^{-1}) \lambda_{\mathrm{min}}(\Sigma_0) \|\theta \|_2^2 \\
        &< 0,
    \end{align*}
    where we used that $\Sigma_0$ is symmetric by construction and non-singular by assumption and therefore positive definite so its smallest eigenvalue $\lambda_\mathrm{min}(\Sigma_0) >0$. Now as $\Gamma \in \mathbb{R}^{p\times p }$ is a strictly negative definite matrix it is non-singular and thus invertible.
    Finally, as $\Omega_0$ and its Schur complement in $M$, $\Gamma$, are invertible, we can employ the standard blockmatrix inversion formula to arrive at the result.
\end{proof}

\paragraph{Proof of Theorem~\ref{th:1:asymptotic-normality}}
The proof generalizes Theorem~3.2 of \citet{newey04} to our regularized continuum estimator.
\begin{proof}
Define $\Psi_{i}(\theta) := \Psi_i(x_i,z_i;\theta)$ and $\hat{\Psi} = \frac{1}{n} \sum_{i=1}^n \Psi_i(\theta)$ and analogous $\Psi^\ast_i(\theta) = \Psi^\ast_i(x_i, z_i; \theta)$ and $\Psi^\ast(\theta) =  \frac{1}{n} \sum_{i=1}^n \Psi^\ast_i(\theta)$.
Let $\hat{\theta}, \hat{h}$ denote the FGEL estimates of the parameters $\theta$ and Lagrange multiplier function $h$. 
The first order optimality conditions of the saddle point objective \eqref{eq:1:gel-objective} are given by
\begin{align}
    D_h \widehat{G}_{\lambda_n}(\hat{\theta}, \hat{h}) &= \frac{1}{n}\sum_{i=1}^n \phi_1(\Psi_i(\hat{\theta})(\hat{h})) \Psi_i^\ast(\hat{\theta}) - \lambda_n \hat{h} = 0 \label{eq:1:first} \\ 
    \nabla_\theta \widehat{G}_{\lambda_n}(\hat{\theta}, \hat{h}) &= \frac{1}{n}\sum_{i=1}^n \phi_1(\Psi_i(\hat{\theta})(\hat{h})) (\nabla_\theta \Psi_i(\hat{\theta}))(\hat{h}) = 0, \label{eq:1:second}
\end{align}
where $\nabla_\theta \Psi_i(\theta) \in \Theta \times \mathcal{H}^\ast$ is the gradient of the function $\theta \mapsto \Psi_i(\theta)$  w.r.t.\ $\theta$.
Define $\beta = (\theta, h)$ then using Taylor's theorem (Proposition~\ref{prop:1:taylor}) we can linearize the first order conditions about the true parameters $\beta_0 = (\theta_0, 0)$ which yields for the first condition \eqref{eq:1:first}
\begin{align*}
    0 =&  - \frac{1}{n} \sum_{i=1}^n \Psi_i^\ast(\theta_0) + \frac{1}{n}\sum_{i=1}^n \phi_2(\Psi_i(\Dot{\theta})(\Dot{h})) \Psi_i^\ast(\Dot{\theta}) \Psi_i(\Dot{\theta})(\hat{h}) \\
    &- \lambda_n \hat{h} + \frac{1}{n}\sum_{i=1}^n \phi_1(\Psi_i(\Dot{\theta})(\Dot{h})) \nabla_{\theta^T} \Psi_i^\ast(\Dot{\theta}) (\hat{\theta} - \theta_0) \\
    &  + \frac{1}{n}\sum_{i=1}^n \phi_2(\Psi_i(\Dot{\theta})(\Dot{h})) \Psi_i^\ast(\Dot{\theta})(\nabla_{\theta^T} \Psi_i(\Dot{\theta}))(\Dot{h}) (\hat{\theta} - \theta_0),
\end{align*}
where $(\Dot{\theta}, \Dot{h})$ lies on the line between $(\hat{\theta}, \hat{h})$ and $(\theta_0, 0)$. For the second condition \eqref{eq:1:second} we obtain
\begin{align*}
    0 =& \frac{1}{n}\sum_{i=1}^n \phi_1(\Psi_i(\bar{\theta})(\bar{h})) (\nabla_\theta \Psi_i(\bar{\theta}))(\hat{h})  \\
    & + \frac{1}{n}\sum_{i=1}^n \phi_2(\Psi_i(\bar{\theta})(\bar{h})) (\nabla_\theta \Psi_i(\bar{\theta}))(\bar{h})  \Psi_i(\bar{\theta})(\hat{h}) \\
    & + \Bigg\{ \frac{1}{n}\sum_{i=1}^n \phi_2(\Psi_i(\bar{\theta})(\bar{h})) (\nabla_\theta \Psi_i(\bar{\theta}))(\bar{h}) (\nabla_\theta \Psi_i(\bar{\theta}))(\bar{h}) \\
    & + \frac{1}{n}\sum_{i=1}^n \phi_1(\Psi_i(\bar{\theta})(\bar{h})) D^2_\theta(\Psi_i(\bar{\theta}))(\bar{h}) \Bigg\} (\hat{\theta} - \theta_0),
\end{align*}
where again $(\bar{\theta}, \bar{h})$ lies on the line between $(\hat{\theta}, \hat{h})$ and $(\theta_0, 0)$.
Now as $\hat{h} = o_p(1)$ we have $\bar{h} = o_p(1)$ and  $ \Dot{h} = o_p(1)$. Therefore for $n\rightarrow \infty$ most terms go to zero and we are left with the first condition \eqref{eq:1:first} reducing to
\begin{align*}
    0 =&  - \frac{1}{n} \sum_{i=1}^n \Psi_i^\ast(\theta_0) + \frac{1}{n}\sum_{i=1}^n \phi_2(\Psi_i(\Dot{\theta})(\Dot{h})) \Psi_i^\ast(\Dot{\theta}) \Psi_i(\Dot{\theta})(\hat{h}) \\
    &- \lambda_n \hat{h} + \frac{1}{n}\sum_{i=1}^n \phi_1(\Psi_i(\Dot{\theta})(\Dot{h})) \nabla_{\theta^T} \Psi_i^\ast(\Dot{\theta}) (\hat{\theta} - \theta_0) \\
    &+ o_p(1) \\
\end{align*}
and the second \eqref{eq:1:second} to
\begin{align*}
    0 =& \frac{1}{n}\sum_{i=1}^n \phi_1(\Psi_i(\bar{\theta})(\bar{h})) (\nabla_\theta \Psi_i(\bar{\theta}))(\hat{h}) + o_p(1).\\
\end{align*}
As $\hat{h} = O_p(n^{-1/2})$ and $\bar{h}, \Dot{h}$ lie between $\hat{h}$ and $0$, the conditions of Lemma~\ref{lemma:1:first} are fulfilled, i.e., $\bar{h}, \Dot{h} \in \mathcal{H}_n$, and hence $\max_{1\leq i \leq n} |\Psi_i(\bar{\theta})(\bar{h})| \overset{p}{\rightarrow} 0$ and $\max_{1\leq i \leq n} |\phi_1(\Psi_i(\bar{\theta})(\bar{h})) + 1| \overset{p}{\rightarrow}0$ as well as $\max_{1\leq i \leq n} |\phi_2(\Psi_i(\bar{\theta})(\bar{h})) + 1| \overset{p}{\rightarrow}0$ and the same equivalently holds for $\Dot{\theta}$ and $\Dot{h}$. 

Define $\hat{\beta} = (\hat{\theta}, \hat{h})$ and $\beta_0 = (\theta_0, 0)$, then, we can write the conditions in matrix form as
\begin{align}
    \begin{pmatrix}
        0 \\ 0
    \end{pmatrix} = \begin{pmatrix} 0 \\ - \frac{1}{n} \sum_{i=1}^n \Psi_i^\ast(\theta_0) \end{pmatrix} + 
    \underbrace{
    \begin{pmatrix}
        0 & - \frac{1}{n} \sum_{i=1}^n \nabla_\theta \Psi_i(\Dot{\theta}) \\ 
        - \frac{1}{n} \sum_{i=1}^n \nabla_{\theta^T} \Psi_i^\ast(\bar{\theta}) &
        - \widehat{\Omega}_{\lambda_n}(\Dot{h}, \Dot{\theta})
    \end{pmatrix} 
    }_{=: M_n}
    (\hat{\beta} - \beta_0) + o_p(1)
    \label{eq2}
\end{align}
Now by the weak law of large numbers and the continuous mapping theorem, we have $\widehat{\Omega}_{\lambda_n}(\Dot{h}, \Dot{\theta}) \overset{p}{\rightarrow} \Omega(\theta_0) =: \Omega_{0}$, which is a non-singular and thus invertible operator. Further for the off-diagonal elements, as $\hat{\theta} \overset{p}{\rightarrow} \theta_0$ and by continuity of $\theta \mapsto \nabla_\theta \Psi(x,z;\theta)$ we have by the continuous mapping theorem and the weak law of large numbers $- \frac{1}{n} \sum_{i=1}^n \nabla_\theta \Psi_i(\Dot{\theta}) \overset{p}{\rightarrow} - E[\nabla_{\theta}\Psi(X,Z;\theta_0)] =: - \nabla_\theta\Psi_0$. Therefore, as $n \rightarrow \infty$ we have $M_n \rightarrow M$ with
\begin{align*}
    M = - \begin{pmatrix}
        0 & \nabla_\theta \Psi_0 \\
        \nabla_{\theta^T} \Psi^\ast_{0} & \Omega_0
    \end{pmatrix}.
\end{align*}
Using Assumptions~e) and k) of Theorem~\ref{th:1:consistency}, it follows from Lemma~\ref{lemma:1:invertible} that the blockoperator $M$ is non-singular and thus invertible with inverse
\begin{align*}
    M^{-1} = -\begin{pmatrix}
        - B & C \\ C^\ast & D
    \end{pmatrix}
\end{align*}
with $B = ((\nabla_{\theta}\Psi_0) \Omega_0^{-1} (\nabla_{\theta^T}  \Psi_0^\ast))^{-1}$, $C = B (\nabla_{\theta}  \Psi_0) \Omega_{0}^{-1} $ and $D = \Omega_{0}^{-1} -   \Omega_{0}^{-1} (\nabla_{\theta^T}  \Psi_0^\ast) B  (\nabla_{\theta}  \Psi_0) \Omega_{0}^{-1} $.

With this at hand we can solve \eqref{eq2} for $\hat{\beta} - \beta_0$ which yields
\begin{align*}
    \sqrt{n} (\hat{\theta} - \theta) &= - \sqrt{n} C  \left( \frac{1}{n} \sum_{i=1}^n \Psi_i^\ast(\theta_0) \right)  + o_p(1) \\
    \sqrt{n} (\hat{h} - h) &= -\sqrt{n} D  \left( \frac{1}{n} \sum_{i=1}^n \Psi_i^\ast(\theta_0) \right) + o_p(1)
\end{align*}
Finally by the Donsker property of $\Psi$ we have $\sqrt{n} \left( \frac{1}{n} \sum_{i=1}^n \Psi(X,Z;\theta_0) \right) \sim \mathcal{N}(0, \Omega_0)$ and thus we get
\begin{align*}
    \sqrt{n} (\hat{\theta} - \theta_0) \sim \mathcal{N}(0, C \Omega_0 C^\ast)
\end{align*}
where 
\begin{align*}
    C \Omega_0 C^\ast &= ((\nabla_{\theta}\Psi_0) \Omega_0^{-1} (\nabla_{\theta^T}  \Psi_0^\ast))^{-1} (\nabla_{\theta}\Psi_0) \Omega_0^{-1} \Omega_0 \Omega_0^{-1} (\nabla_{\theta^T}\Psi_0^\ast) ((\nabla_{\theta}\Psi_0) \Omega_0^{-1} (\nabla_{\theta^T}  \Psi_0^\ast))^{-1} \\ 
    &= (\nabla_{\theta}\Psi_0) \Omega_0^{-1} (\nabla_{\theta^T}  \Psi_0^\ast))^{-1}
\end{align*}

\end{proof}
\subsubsection{Asymptotic Properties for CMR}
The asymptotic properties of the FGEL estimator for conditional moment restrictions follow by expressing the conditional moment restrictions \eqref{eq:1:conditional} in terms of the equivalent variational/functional formulation \eqref{eq:1:unconditional} and translating the assumptions of Theorems~\ref{th:1:consistency-cmr} and \ref{th:1:asymptotic-normality-cmr} into the conditions required for Theorems~\ref{th:1:consistency} and \ref{th:1:asymptotic-normality} respectively. The proofs are almost identical to the ones provided by \citet{pmlr-v202-kremer23a} for translating the results on their KMM estimator for functional moment restrictions to the conditional moment restriction case. Therefore, we only state the differences here and refer to the proofs of Theorems~3.4-3.6 of \citet{pmlr-v202-kremer23a} for details.

\paragraph{Proof of Theorem~\ref{th:1:consistency-cmr}}
The proof follows directly from the proof of Theorem~3.4 of \citet{pmlr-v202-kremer23a}, ignoring their assumption f) which is not required here and instead using that the additional Assumption~g) of the FGEL estimator is fulfilled by the identical Assumption~g) in Theorem~\ref{th:1:consistency}. Further the Donsker property h) of Theorem~\ref{th:1:consistency} is fulfilled by the corresponding Assumption~h) and Lemma~\ref{lemma:1:donsker} using that both $\psi$ and $h$ are uniformly bounded as continuous functions on compact domains.

\paragraph{Proof of Theorem~\ref{th:1:asymptotic-normality-cmr}}
The proof is identical to the one provided by \citet{pmlr-v202-kremer23a} for their Theorem~3.5.

\paragraph{Proof of Theorem~\ref{efficiency-cmr}}
Efficiency follows immediately from Theorem~\ref{th:1:asymptotic-normality-cmr} as the asymptotic variance of the FGEL estimator agrees with the semi-parametric efficiency bound for CMR estimators by \citet{CHAMBERLAIN1987305}.

\subsection{Kernel FGEL}
\paragraph{Proof of Theorem~\ref{th:1:equivalence}}
\begin{proof}
The proof follows the proof of Theorem~3.2 in \citet{muandet2020kernel}.
Equation~\eqref{eq:1:unconditional-rkhs} follows from \eqref{eq:1:conditional-rkhs} directly by the law of iterated expectation. To see this, assume $E[\psi(X;\theta)|Z] = 0 \  P_Z\textrm{-a.s.}$, then $\forall h \in \mathcal{H}$
\begin{align*}
    E[\Psi(X,Z;\theta)(h)] &=  E [ \psi(X;\theta) h(Z) ] \\
    &=  E [ E[ \psi(X;\theta) h(Z) |Z ]] \\
    &= E [ E[ \psi(X;\theta) |Z] h(Z) ] \\
    &= 0.
\end{align*}
For the other direction, 
note that $E[\Psi(X,Z;\theta)(h)]=0$ $\forall h \in \mathcal{H}$ implies $\sup_{h \in \mathcal{H}} E[\Psi(X,Z;\theta)(h)] = 0$ and thus
\begin{align*}
    0 & = \sup_{h \in \mathcal{H}} E[\Psi(X,Z;\theta)(h)]  \\
    &= \sum_{j=1}^m \sup_{\|h_j\|_{\mathcal{H}} \leq 1} E[\psi_j(X;\theta) h_j(Z)]  \\
    &= \sum_{j=1}^m \sup_{\|h_j \|_{\mathcal{H}} \leq 1}   \langle E[\psi_j(X;\theta) k_{j}(Z,\cdot)], h_j\rangle \\
    &= \sum_{j=1}^m \| E[\psi_j(X;\theta) k_{j}(Z,\cdot)]\|_{\mathcal{H}} \\
    &= \sum_{j=1}^m \| E_Z[ \underbrace{E_X[ \psi_j(X;\theta)| Z]}_{:= \xi_j(Z)}  k_j(Z,\cdot) ] \|_{\mathcal{H}}  \\
    &= \sum_{j=1}^m \| \int_{\mathcal{Z}} \xi_j(z) k_j(z,\cdot) p(z) dz \|_{\mathcal{H}}
\end{align*}
As each element of the sum is non-negative, we must have for $j=1,\ldots,m$,
\begin{align*}
0 &= \| \int_{\mathcal{Z}} \xi_j(z) k_j(z,\cdot) p(z) \mathrm{d}z \|_{\mathcal{H}} \\
&= \| \int_{\mathcal{Z}} \xi_j(z) k_j(z,\cdot) p(z) \mathrm{d}z \|_{\mathcal{H}}^2 \\
&= \int_{\mathcal{Z} \times \mathcal{Z}} \xi_j(z) \langle k_j(z,\cdot), k_j(z',\cdot)\rangle_{\mathcal{H}} \xi_j(z')  p(z) p(z') \mathrm{d}z \mathrm{d}z' \\
&= \int_{\mathcal{Z} \times \mathcal{Z}} \xi_j(z) k_j(z,z') \xi_j(z') p(z) p(z') \mathrm{d}z \mathrm{d}z'. 
\end{align*}
By definition of ISPD kernels (see Section~\ref{sec:1:FGEL}) this directly implies  $\| \xi_j(z) p(z)\|_2^2= 0$. It follows that $\xi_j(z)=0 \ \ \mathrm{a.e.}$ on the support of $p(z)$ and thus $P_Z(\{ z \in \mathcal{Z} : \xi_j(z)=0 \}) = 1$. Finally this implies 
$$\xi_j(Z) = E[ \psi_j(X;\theta)| Z] = 0 \ \ P_{Z}\textrm{-a.s.}, \ \ j=1,\ldots,m,$$
which completes the equivalence between \eqref{eq:1:conditional-rkhs} and \eqref{eq:1:unconditional-rkhs}.
\end{proof}
\paragraph{Proof of Corollary~\ref{cor:consistency-rkhs}}
\begin{proof}
    Equivalence between \eqref{eq:1:conditional} and \eqref{eq:1:unconditional} holds for any RKHS corresponding to a universal ISPD kernels by Theorem~\ref{th:1:equivalence}. Moreover, the local Lipschitz property is fulfilled as in any RKHS the evaluation functional is bounded, which implies that for any $h \in \mathcal{H}$
    \begin{align*}
        \| h(z_1) - h(z_2) \|  &= \| \langle h, k(z_1, \cdot) \rangle - \langle h, k(z_2, \cdot) \rangle \| \\
        &\leq \| h \| \|k(z_1,\cdot) - k(z_2,\cdot) \| \\
        &\leq C \left( \|k(z_1,\cdot) \|_\mathcal{H} + \|k(z_2, \cdot)  \|_\mathcal{H}\right) \\
        &\leq L
    \end{align*}
    where we used the Cauchy-Schwarz and triangle inequalities. Finally the Donsker property of $\mathcal{H}_1$ follows from Lemma~17 of \citet{bennett2020variational}.
\end{proof}
\paragraph{Proof of Lemma~\ref{th:1:optimization-problem}}
\begin{proof}
    The profile divergence can be written as 
    \begin{align*}
        R_{\lambda_n}(\theta) = \inf_{h \in \widehat{\mathcal{H}}} - \sum_{i=1}^n \phi( \Psi(x_i,z_i;\theta)(h) + \frac{\lambda_n}{2} \|h \|_{\mathcal{H}}.
    \end{align*}
    As $-\phi$ is a convex function and $\widehat{\mathcal{H}}$ is convex it follows that this is a convex optimization problem. Therefore, we can employ the representer theorem \citet{Schoelkopf01:Representer} and express each component $r$ of the $m$-dimensional vector of RKHS functions as $h_r (\cdot)=\sum_{i=1}^n (\alpha_{r})_i k_r(z_i,\cdot)$, with $\alpha_r \in \mathbb{R}^n$. Therefore, we get 
    \begin{align*}
        &\Psi(x_i,z_i;\theta)(h) = \sum_{r=1}^m \sum_{i,j=1}^n (\alpha_r)_j \left(K_r\right)_{ji} \psi_r(x_i,z_i;\theta)\\
        &\| h \|_{\mathcal{H}}^2 = \sum_{r=1}^m \sum_{i,j=1}^n (\alpha_{r})_i \langle k_r(z_i,\cdot) , k_r(z_j,\cdot) \rangle  (\alpha_{r})_j = \sum_{r=1}^m \alpha_r^T K_r \alpha_r
    \end{align*}
    Inserting this back into $R_{\lambda_n}(\theta)$ yields the result.
\end{proof}
\subsection{Additional Proofs}
\paragraph{Proof of Proposition~\ref{prop:1:equivalence}}
\begin{proof}
The result follows directly by inserting $\phi(v)= \left(1 \pm \frac{v}{2}\right)^2$ into \eqref{eq:1:gel-objective} and using that as $\mathcal{H}$ is a vector space, for every $h \in \mathcal{H}$, its negative $-h$ is also contained in $\mathcal{H}$. Therefore the first order conditions agree for the positive and negative sign in $\phi$.
\end{proof}
\end{document}